\documentclass[12pt,english]{article}
\usepackage[utf8]{inputenc}
\usepackage[a4paper]{geometry}
\usepackage{babel}
\usepackage{longtable}
\usepackage{mathrsfs}
\usepackage{mathtools}
\usepackage{amsmath}
\usepackage{amsthm}
\usepackage{amssymb}
\usepackage[unicode=true,
 bookmarks=true,bookmarksnumbered=true,bookmarksopen=false,
 breaklinks=false,pdfborder={0 0 1},backref=false,colorlinks=true]
 {hyperref}
\hypersetup{
 linkcolor=blue,citecolor=[rgb]{0,0.6,0},urlcolor=blue}

\usepackage{bbm}

\makeatletter

\numberwithin{equation}{section}
\theoremstyle{definition}
\newtheorem*{defn*}{\protect\definitionname}
\theoremstyle{plain}
\newtheorem{thm}{\protect\theoremname}[section]
\theoremstyle{plain}
\newtheorem{lem}[thm]{\protect\lemmaname}
\theoremstyle{plain}
\newtheorem*{thm*}{\protect\theoremname}
\theoremstyle{plain}
\newtheorem*{lem*}{\protect\lemmaname}
\theoremstyle{definition}
\newtheorem{defn}[thm]{\protect\definitionname}
\theoremstyle{plain}
\newtheorem{prop}[thm]{\protect\propositionname}
\theoremstyle{plain}
\newtheorem{cor}[thm]{\protect\corollaryname}
\theoremstyle{plain}
\newtheorem{assumption}[thm]{\protect\assumptionname}
\theoremstyle{definition}
\newtheorem{rem}[thm]{Remark}

\usepackage{authblk}
\author[1,2,3]{Hayk Asatryan\thanks{asatryan@math.uni-wuppertal.de}}
\author[1,2,3]{Hanno Gottschalk\thanks{hanno.gottschalk@uni-wuppertal.de}}
\author[1]{Marieke Lippert\thanks{marieke-lippert@gmx.de}}
\author[1,2,3]{Matthias Rottmann\thanks{rottmann@math.uni-wuppertal.de}}
\affil[1]{\small School of Mathematics and Natural Sciences, University of Wuppertal, Germany}
\affil[2]{\small Institute of Mathematical Modelling, Analysis and Computational Mathematics (IMACM)}
\affil[3]{\small Interdisciplinary Center for Machine Learning and Data Analytics (IZMD)}

\makeatother

\providecommand{\assumptionname}{Assumption}
\providecommand{\corollaryname}{Corollary}
\providecommand{\definitionname}{Definition}
\providecommand{\lemmaname}{Lemma}
\providecommand{\propositionname}{Proposition}
\providecommand{\theoremname}{Theorem}

\usepackage[backend=biber,style=numeric,giveninits=true,url=false,isbn=false,doi,related=false, maxbibnames=99]{biblatex}
\usepackage{tikz}
\usepackage{tkz-euclide}
\usetikzlibrary{matrix,chains,positioning,decorations.pathreplacing,arrows, calc, fit, through, shapes.geometric}
\tikzstyle{arrow} = [thick,->,>=stealth]
\usetikzlibrary{shadows}

\definecolor{lgray1}{rgb}{0.83, 0.83, 0.83}
\definecolor{lgray2}{rgb}{0.9, 0.9, 0.9}
\definecolor{navy}{rgb}{0.0, 0.0, 0.5}
\definecolor{bblue}{rgb}{0.74, 0.83, 0.9} 
\definecolor{egreen}{rgb}{0.0, 0.5, 0.0}
\definecolor{bred}{rgb}{0.8, 0.0, 0.0}

\allowdisplaybreaks 

\begin{document}
\title{\vspace{-2ex}
A Convenient Infinite Dimensional Framework\linebreak{}
for Generative Adversarial Learning}

\maketitle
\begin{longtable}[c]{p{0.9\textwidth}}
\textbf{\small{}Abstract. }{\small{}In recent years, generative adversarial networks (GANs) have
demonstrated impressive experimental results while there are only
a few works that foster statistical learning theory for GANs. In this
work, we propose an infinite dimensional theoretical framework for
generative adversarial learning. We assume that the probability density functions of the underlying measure are uniformly bounded,
$k$-times $\alpha$-H\"{o}lder differentiable ($C^{k,\alpha}$) and
uniformly bounded away from zero. Under these assumptions, we show that the Rosenblatt transformation
induces an optimal generator, which is realizable in the hypothesis
space of $C^{k,\alpha}$-generators. With a consistent definition
of the hypothesis space of discriminators, we further show that the
Jensen-Shannon divergence between the distribution induced by
the generator from the adversarial learning procedure and the data
generating distribution converges to zero. Under certain regularity
assumptions on the density of the data generating process, we also
provide rates of convergence based on chaining and concentration.}{\small\par}

\smallskip{}

\textbf{\small{}Keywords:}{\small{} Generative Adversarial Learning
$\bullet$ Inverse Rosenblatt Transformation $\bullet$ Statistical
Learning Theory $\bullet$ Chaining $\bullet$ Covering Numbers for
Hölder Spaces}{\small\par}

\smallskip{}

\textbf{\small{}MSC 2020:}{\small{} 62G20, 68T05}\tabularnewline
\end{longtable}

\smallskip{}

\section{Introduction}

\addcontentsline{toc}{section}{Introduction}

Generative learning aims at modeling the distribution of a data generating
process and sampling from it. This desire is at least as old as Markov
Chain Monte Carlo (MCMC) methods \cite{hastings70} and has resulted in
several different types of models and methods such as, to name only
a few of them, hidden Markov models (HMM), Gaussian mixture models
(GMM), Boltzmann machines (BM), principal component analysis (PCA),
(variational) autoencoders (AE) and generative adversarial networks (GAN), see \cite{Ghahramani2004,HintonBoltzmann,Goodfellow14,kingma2019introduction}.
Most of these models either aim at dimensionality reduction (PCA,
AE) but do not allow for sampling, or allow for sampling but suffer
from the curse of dimensionality (GMM, BM) or allow sampling in high
dimension but require an explicit density that has to be constructed
or learned beforehand (MCMC). However, this changed with the advent
of GANs \cite{Goodfellow14} that can be classified under the more
general notion of adversarial learning. Eversince, GANs have demonstrated
remarkable capabilities in the domain of generative modeling, in particular,
with application to image synthesis \cite{radford2016unsupervised,ledig2017photo,brock2018large,Zhu.2018}.

In the typical adversarial learning framework, a generator is supposed
to map examples from a noise distribution, say the uniform distribution
over a compact set, to examples of a desired distribution, a sample
of real / original data of which is available. Mathematically this
corresponds to a pushforward or image measure of the noise measure
under the generator map. This map is learned in a two player setting:
A discriminator is trained to distinguish between the real examples
from a data set and the fake ones, i.e., those obtained by the generator.
On the other hand, the generator is trained to fool the discriminator.
Learning both the discriminator and the generator alternatingly /
simultaneously leads to a generator producing ever more realistic
examples that approximately follow the distribution the data stems
from.

GANs utilize deep neural networks (DNNs) as generators and discriminators.
They allow for sampling from a model of the data generating distribution
without the necessity of explicitly modeling densities (such as MCMC,
GMM), without critical slowing down of learning (MCMC, BM) and barely
suffering from the curse of dimensionality.

Many efforts have been made to empirically improve the GAN framework for image synthesis \cite{radford2016unsupervised,ledig2017photo,brock2018large}. A remarkable improvement was obtained by the so-called
CycleGAN framework \cite{Zhu.2018}. Therein, not only the generator
but also its inverse is learned (with a corresponding second discriminator).
Impressive empirical results are presented, e.g., for transforming
horses into zebras and vice versa. Therefore, GANs have also demonstrated
to cope with higher resolution image data.

While the loss functions used in many works aim at reducing the Jensen-Shannon
(JS) divergence between the true distribution and the class of generated
(or parameterized) ones, other loss functions aiming at reducing other
distance metrics have also been proposed, e.g., Wasserstein-GANs \cite{pmlr-v70-arjovsky17a}
that aim at reducing the Wasserstein distance. Wasserstein-GANs may
in some cases yield better empirical convergence properties, where
ordinary GANs suffer from so-called mode collapsing \cite{lala2018evaluation}.

In contrast to the rapidly progressing developments on the empirical
side, theoretical results on generative adversarial learning\footnote{Here we use the term generative adversarial learning in contrast to
generative adversarial networks if we do not refer to neural networks
as models, specifically.} remained unexamined until the publication \cite{biau2020}. That
work studies the connection between the adversarial principle of generative
adversarial learning and the Jensen-Shannon divergence in a framework
with a finite dimensional hypothesis space. This includes existence
and uniqueness arguments for the optimal generator. While the discriminator
is modeled in a rather abstract way, the authors also provide approximation
arguments. In particular, they provide large sample theory for generative
adversarial learning as well as a central limit theorem. However this
is restricted to finite dimensional parametric models.

Adversarial learning has also been applied to 
density estimation
in the nonparametric context \cite{singh2018,uppal2020nonparametric}.
Universal approximation theorems for deep neural networks \cite{yarotsky2017error}
are used to approximate densities. This is however very different
from \cite{Goodfellow14} and also from this work, as no generator
for sampling from noise input is learned. In \cite{singh2018} generators
are considered, however the existence of such generators in certain
Besov spaces is assumed (cf. Assumption A4 in Theorem 9), whereas
in this article we give proofs and an explicit construction for H\"{o}lder
spaces.

Further theoretical work on generative adversarial learning that was
recently published includes theory on Wasserstein methods \cite{biau2020wasserstein},
as well as theory on domain shifts quantified by means of an adversarial
loss that reduces Jensen-Shannon divergence \cite{shui2020mathcalhdivergence}.

In this work we extend the approach of \cite{biau2020} to an infinite
dimensional setting where the generators are $k$-times differentiable
$\alpha$-H\"{o}lder ($C^{k,\alpha}$-) functions defined on $[0,1]^{d}$
(e.g., the input space of images with color intensities in $[0,1]$
for image generation tasks). The discriminator space is chosen consistently,
such that optimal discriminators always exist. This enables us to
prove the existence of an optimal generator under quite general assumptions
on the probability density of the `true' data generating process on
Borel sets of $[0,1]^{d}$. This convenient, but not very restrictive
realizability property of our framework avoids the use of approximations
as in \cite{biau2020}. We achieve this proving the Rosenblatt transformation
\cite{ros52} to be in the hypothesis space for the generators consisting
of bounded, invertible $C^{k,\alpha}$-functions. In this way, we
can estimate the error, measured in terms of the Jensen-Shannon divergence,
between the probability distribution from the generative adversarial
learning process and the `true' distribution by the sampling error,
i.e., the supremum over the empirical process over the product of
the hypothesis spaces of generators and discriminators \cite{shalev2014understanding}.

Statistical learning theory to a large extent depends on compactness
properties of the hypothesis space. In infinite dimensional statistics,
where the hypotheses are parameterized by functions in bounded regions
of some -- say Banach -- function space \cite{Vaart96,gine2016mathematical,Handel16},
the use of the Banach topology is prohibited by Riesz's theorem which
characterizes locally compact Banach spaces (see, e.g., \cite{Dunf88vol1,rud91functional}).
Therefore, compact embeddings into spaces with weaker topology play
a crucial role in infinite dimensional statistics. In our convenient
framework, we obtain such embeddings from the embedding of bounded
$C^{k,\alpha}$-functions into $C^{k,\alpha'}$-functions for $0<\alpha'<\alpha<1$
(see \cite{Gil01}). With the aid of the uniform law of large numbers
over compact spaces \cite{Ferguson.2017}, we can thus conclude that
in this setting of adversarial learning, generators rendering the
true distribution in the large sample limit can always be learned.

Under stronger regularity assumptions, assuming sufficiently high H\"{o}lder regularity of the density $k+\alpha$, this statement can also be
made quantitative in order to prove explicit rates of convergence $\sim n^{-\frac{1}{2}}$ up to polylogarithms. Under the same assumptions as ours, \cite{chen2020distribution} derived a rate $n^{-\frac{k+\alpha}{2(k+\alpha)+d}}$ (up to polylogarithms) of convergence for a Wasserstein-like $\mathcal{F}$-divergence based on H\"{o}lder functions, which, in the limit $k\to\infty$, matches our rates. However, this $\mathcal{F}$-divergence is weaker  than the Jensen-Shannon divergence employed here. On the other hand,  \cite{chen2020distribution} invoke neural networks with architectures adaptive to the sample size, whereas here we restrict ourselves to an infinite dimensional statistics setting. In a recent preprint, \cite{belomnestry2021} suggest rates $\sim n^{-1}$ (up to polylogarithms) in the limit of arbitrarily high regularity.

The key observations are that (a) for the embedding of bounded $C^{k,\alpha}$-functions
into $L^{\infty}\left([0,1]^{d}\right)$ covering numbers are explicitly
known and they allow a convergent metric entropy integral in Dudley's
inequality, provided the regularity defined by $k+\alpha$ being sufficiently
high, and (b) the fact that the empirical process defined by the empirical
loss function is a subgaussian process with respect to the $\|\cdot\|_{\infty}$-norm.
Thus, explicit estimates for the supremum of the empirical process
/ the sampling error can be obtained via chaining and concentration
(see \cite{Vaart96,McDiar89}). These rates, in contrast to \cite{biau2020},
do not depend on the dimension of the hypothesis space (which is infinite
in our setting). We also give an argument how to adaptively extend
the hypothesis space with the sample size in order to eliminate certain
assumptions on the regularity of the probability density of the data
generating process and achieve almost sure convergence of the generative
adversarial learner in Jensen-Shannon distance, also for this case.

While this article is motivated by generative adversarial learning, our results on the existence of generators might be of independent interest in the fields of copula theory \cite{joe2014dependence}, normalizing flows \cite{dinh2014nice,kobyzev2020normalizing} or optimal transport \cite{brenier1991polar,villani2009optimal}.
\medskip{}

\noindent \textbf{Outline. }This paper is organized as follows: In
Section~\ref{sec:outl_glearn} we introduce the framework of adversarial learning \cite{Goodfellow14} and review the connection of its loss function to the Jensen-Shannon divergence. We then make use of the latter to state a decomposition of the estimation error.
Thereafter in Section~\ref{Sec:Gener_and_Discr}, first state our assumptions on the target measure, before we introduce the hypothesis
class of $k$-times $\alpha$-H\"{o}lder differentiable generators
with the properties outlined above and prove that if the `true' data
generating process has a nonnegative $C^{k,\alpha}$-density function,
the optimal generator given by the Rosenblatt transformation is contained
in the hypothesis space. Furthermore, we introduce the consistent
class of discriminators that also contains the optimal one. Furthermore, we prove the uniform convergence of
the empirical loss to its expected value, which implies that the sampling
error vanishes in the limit. We also prove the convergence of the
probability distribution, generated by the empirical risk minimizer
in the min-max problem from the adversarial learning setting, to the
data generating distribution in Jensen-Shannon divergence. Under stronger
regularity assumptions on the `true' density, in Section \ref{sec:quantitative}
we give explicit rate estimates based on covering numbers for the embedding 
of the hypothesis space into $L^{\infty}\left([0,1]^{d}\right)$,
which is used to eliminate certain regularity assumptions. We give
a summary and outline future research directions in the final Section
\ref{sec:CO}. Technical results on the analysis of H\"{o}lder spaces can be found in Appendix \ref{app:A}, whereas Appendix \ref{app:B} collects some inequalities from high dimensional probability theory for the convenience of the reader.

\section{An Outline of Generative Adversarial Learning}
\label{sec:outl_glearn}

In this section, we give an overview into the general framework of generative adversarial learning following \cite{Goodfellow14}. Generative learning aims learning a process that samples from an unknown original measure $\mu$ from data $Y_i$ sampled from this measure ($Y_i \sim \mu$). In practice, generative learning is oftentimes utilized for sake of generating realistic images $y \in [0,1]^{d}$, $d \in \mathbb{N}$. In that context, let $\mu$ be a given probability measure on the measurable space $\left([0,1]^{d},\mathscr{B}[0,1]^{d}\right)$ with $\mathscr{B}[0,1]^{d}$ being the Borel-$\sigma$-algebra, and let $\lambda$ denote the restriction of the $d$-dimensional Lebesgue measure to $\left([0,1]^{d},\mathscr{B}[0,1]^{d}\right)$. 

A strategy that is oftentimes applied in generative learning is to learn a so-called \emph{generator}, namely a measurable function $\varphi:[0,1]^{d}\to[0,1]^{d}$ from a \emph{hypothesis space} of candidate functions $\mathscr{H}^G$, such that the distribution $\mu_\varphi=\varphi_*\lambda$ of $\varphi(Z)$ approximates the original measure $\mu$ and $Z$ is easy to sample. Here, $\varphi_*\lambda(A)=\lambda(\varphi^{-1}(A))$, $A\in\mathscr{B}[0,1]^d$, stands for the pushforward or image measure of $\lambda$ under $\varphi$. Ideally, $Z$ is a noise random variable $Z \sim \lambda$ with $\lambda$ being the Lebesgue measure. Simulation of $Z$ allows then for generating data $\varphi(Z)$ from the approximate distribution $\mu_{\varphi}$ that resembles the original data $Y$ drawn from $\mu$.

\begin{figure}[t]
\centering
\scalebox{0.7}{
\begin{tikzpicture}
\node[draw,ultra thick, rectangle,fill=lgray2, align=center, minimum width=1cm, minimum height=1cm, rounded corners] (z) at (0,0)  {\small random\\ \small variables $Z_i$};
\node[draw=navy, ultra thick, rectangle, fill=bblue, fill opacity=0.6, text opacity=1, align=center, minimum width=2.5cm, minimum height=2cm, rounded corners] (g) at (3,0) {\small Generator $\varphi$};
\draw[->, thick](z)--(g);
\node[draw, ultra thick, rectangle,fill=lgray2, align=center, minimum width=1.5cm, minimum height=1.5cm, rounded corners] (fs) at (6,0)  {\small generated\\ \small sample\\ \small $\varphi(Z_i)$};
\draw[->, thick](g)--(fs);


\node[draw,ultra thick, rectangle,fill=lgray2, align=center, minimum width=1.5cm, minimum height=1.5cm, rounded corners] (rs) at (6,2.5)  {\small original\\ \small sample\\ \small $Y_i$};
\node[draw=navy, ultra thick, rectangle, fill=bblue, fill opacity=0.5, text opacity=1, align=center, minimum width=2.5cm, minimum height=2cm, rounded corners] (d) at (9,1.2) {\small Discriminator $\xi$};
\coordinate (dp) at ([xshift=-0.15cm]d.west);
\draw[->, thick](fs) .. controls (7, 0.25) and (7.2,0.5) .. (dp)-- (d);
\draw[->, thick](rs) .. controls (7, 2.25) and (7.2,2.25) .. (dp)-- (d);
\node[draw, ultra thick, rectangle, fill=lgray2, align=center, minimum width=1.6cm, minimum height=2cm, rounded corners] (o) at (12,1.2) {};
\node[draw=lgray1, circle, fill=egreen, align=center, label={[label distance=0.01cm, align=center]90:\small \textcolor{egreen}{original}}] at(12, 1.4){};
\node[draw=lgray1, circle, fill=bred, align=center, label={[label distance=0.01cm, align=center]270:\small \textcolor{bred}{generated}}] at(12, 0.9){};
\draw[->, thick](d)--(o);
\node[draw,ultra thick, rectangle,fill=lgray2, align=center, minimum width=1cm, minimum height=0.5cm, rounded corners] (l) at (14.5,1.2)  {\small objective $\hat{L}$};
\draw[->, thick](o)--(l);
\end{tikzpicture}}
\caption{\small\itshape Illustration of the adversarial learning framework studied in the present manuscript. The discriminator aims at tagging samples from $Y_i$ as original and generated samples from $\varphi(Z_i)$ as generated, based on the estimated probabilities $\xi(\varphi(Z_i))$ and $\xi(Y_i)$. Blue components of the figure are the learnable components subject to optimization.}
\label{fig:gan}
\end{figure}

Adversarial learning compares original data obtained by $Y$ with generated data obtained by $\varphi(Z)$ and the learning process (performed by stochastic optimization) modifies $\varphi$ such that examples of $\varphi(Z)$ become indistinguishable from generated examples obtained by $\varphi(Z)$. 
This is achieved with the help of a so-called \emph{discriminator}, a function $\xi: [0,1]^{d} \to (0,1)$  which is supposed to estimate the probability of its input stemming from the original measure $\mu$ vs.\ its input stemming from $\mu_\varphi$. As a suitable discriminator is not known \emph{a priori}, it has to be learned from the data by maximizing the log-likelihood 
\begin{equation}
\label{eq:likelihoodDiscriminator}
\hat{L}\left(\varphi,\xi,n\right)=\dfrac{1}{2n}\sum_{i=1}^{n}\log\xi\left(Y_{i}\right)+\dfrac{1}{2n}\sum_{i=1}^{n}\log\left[1-\xi\left(\varphi\left(Z_{i}\right)\right)\right] 
\end{equation}
 over a hypothesis space of discriminators $\mathscr{H}^D$. Here and in the following $Y_i$ are i.i.d.\ copies of $Y$ and $Z_i$ i.i.d.\ copies of $Z$. The generative adversarial learning problem (see also Fig. \ref{fig:gan}) then amounts to the solution of the minimax problem 
\begin{equation} \label{eq:al_Lhat}
\min\limits _{\varphi\in\mathscr{H}^{G}}\max\limits _{\xi\in\mathscr{H}^{D}} \hat{L} \left(\varphi,\xi\right).
\end{equation}
The advantage of generative adversarial learning lies in the fact that \eqref{eq:likelihoodDiscriminator} avoids the representation of $\mu_\varphi$ by its density, which is hard to compute, unless one employs restrictive neural network architectures \cite{teshima2020coupling}.

Despite generative adversarial learning is intuitive, a deeper theoretical understanding is desirable to explain why it actually works. To this aim, we have to define what it means to have only a small difference between $\mu_\varphi$ and $\mu$.   
To measure the dissimilarity between two probability distributions,
we will use the Kullback-Leibler and the Jensen-Shannon divergences
(see \cite{Kul97}). These divergences will
be helpful to evaluate a learned measure imitating the given one.

Let $\nu$ and $\mu$ be two equivalent (i.e., absolutely continuous
with respect to one another) probability measures on $[0,1]^{d}$. Recall that the \emph{Kullback-Leibler
divergence}  is defined by the formula
$d_{KL}\left(\nu||\mu\right)\coloneqq\int\log\tfrac{d\nu}{d\mu}d\nu$,
provided that the integral on the right hand side exists (here $\tfrac{d\nu}{d\mu}$
is the Radon-Nikodym derivative). Jensen's integral inequality shows
that the integral $\int\log\tfrac{d\nu}{d\mu}d\nu$ exists in the extended sense and it is nonnegative. Moreover, $d_{KL}\left(\nu||\mu\right)$
vanishes iff $\nu=\mu$. Here and in the following the integration is over $[0,1]^d$ if not specified.

Since the Kullback-Leibler divergence is not symmetric in general,
we will use another measure of dissimilarity between two probability
distributions, called the Jensen-Shannon divergence.
\begin{defn}
The \emph{Jensen-Shannon
divergence} is defined by the formula
\[
d_{JS}\left(\nu,\mu\right)\coloneqq\tfrac{1}{2}\left[d_{KL}\left(\mu\left\Vert \tfrac{\mu+\nu}{2}\right.\right)+d_{KL}\left(\nu\left\Vert \tfrac{\mu+\nu}{2}\right.\right)\right].
\]
\end{defn}

In the following we relate \eqref{eq:likelihoodDiscriminator} and the Jensen-Shannon divergence. 
By the law of large numbers, under suitable conditions \eqref{eq:likelihoodDiscriminator} almost surely converges to  
\begin{equation}
\label{eq:theor_loss}
L\left(\varphi,\xi\right)=\frac{1}{2}\left(\mathbb{E}_{Y\sim\mu}\left[\log\left(\xi\left(Y\right)\right)\right]+\mathbb{E}_{Z \sim \lambda}\left[\log\left(1-\xi\left(\varphi(Z)\right)\right)\right]\right) .    
\end{equation}
Thus, the minimax optimization problem \eqref{eq:al_Lhat} can be seen as an approximation to the minimax problem for $L\left(\varphi,\xi\right)$, namely
\begin{equation} \label{eq:al_L}
\min\limits _{\varphi\in\mathscr{H}^{G}}\max\limits _{\xi\in\mathscr{H}^{D}} L \left(\varphi,\xi\right) .
\end{equation}
The following proposition relates the solution of \eqref{eq:al_L} to the search of $\mu_\varphi$ which is closest to $\mu$ in terms of the Jensen-Shannon divergence (cf. \cite{Goodfellow14}).

\begin{prop} \label{prop:Dis_JS}
Suppose $\mu$ is absolutely continuous with respect to the Lebesgue measure $\lambda$ and its density $f_\mu$ is positive on $[0,1]^d$. Additionally, we assume that $\mu_\varphi$ for $\varphi\in \mathscr{H}^G$ also fulfills the same condition with associated density $f_\varphi>0$. Finally, we assume that for $\varphi\in\mathscr{H}^G$, $\xi_\varphi=\frac{f_\mu}{f_\mu+f_\varphi}\in \mathscr{H}^D$ holds. 

Then, $\xi_\varphi$ is the unique solution (up to redefinition on  null sets of $\lambda$) to
\[
\max\limits_{\xi\in\mathscr{H}^D}L(\varphi,\xi)=d_{JS}(\mu,\mu_\varphi)-\log(2).
\]
\end{prop}
\begin{proof}
By setting
\[
h\left(s,r,p\right)\coloneqq\dfrac{1}{2}\left[s\log p+r\log\left(1-p\right)\right]\quad\left(s,r>0;0<p<1\right),
\]
we rewrite \eqref{eq:theor_loss} in the form
\[
L\left(\varphi,\xi\right)=\int h\left(f_{\mu},f_{\varphi},\xi\right)d\lambda \quad\left(\varphi\in\mathscr{H}^{G},\xi\in\mathscr{H}^{D}\right).
\]
It is easy to check that for fixed $s$ and $r$ the function $h\left(s,r,\cdot\right)$
attains its strict global maximum at $p=\frac{s}{s+r}$. Hence, in
view of the assumption $\xi_{\varphi}\in\mathscr{H}^{D}$, we have 
\[
h\left(f_{\mu},f_{\varphi},\xi\right)\leq h\left(f_{\mu},f_{\varphi},\xi_{\varphi}\right)\quad\left(\varphi\in\mathscr{H}^{G},\xi\in\mathscr{H}^{D}\right).
\]
which shows that $\xi_{\varphi}=\tfrac{f_{\mu}}{f_{\mu}+f_{\varphi}}$
maximizes $L\left(\varphi,\cdot\right)$. Moreover, if  $E\coloneqq\left\{ x:\xi(x)\ne\xi_{\varphi}(x)\right\} $ is
a set of positive Lebesgue measure, it follows that
\[
\left.h\left(f_{\mu},f_{\varphi},\xi\right)\right|_{E}<\left.h\left(f_{\mu},f_{\varphi},\xi_{\varphi}\right)\right|_{E};
\]
and therefore $L\left(\varphi,\xi\right)<L\left(\varphi,\xi_{\varphi}\right)$, which shows that $\xi$ is no maximizer.

Finally, writing the Radon-Nikodym derivative as a quotient of densities, we obtain 
\begin{align*}
L\left(\varphi,\xi_{\varphi}\right)&=\frac{1}{2}\left[\int \log\left(\frac{f_\mu}{f_\mu+f_\varphi}\right) f_\mu\, d\lambda+\int \log\left(\frac{f_\varphi}{f_\mu+f_\varphi}\right) f_\varphi\, d\lambda\right]\\
&=\frac{1}{2}\left[\int \log\left(\frac{2f_\mu}{f_\mu+f_\varphi}\right) f_\mu\, d\lambda+\int \log\left(\frac{2f_\mu}{f_\mu+f_\varphi}\right) f_\varphi\, d\lambda\right]-\log(2)\\
&=d_{JS}(\mu,\mu_\varphi)-\log(2).\qedhere
\end{align*}
\end{proof}

In the next step, we compare the solution of the accessible minimax problem \eqref{eq:al_Lhat} to the actual minimization of the Jensen-Shannon divergence by solution of \eqref{eq:al_L}. This leads to a decomposition of the estimation error made in generative adversarial learning. A similar estimate can be found in \cite{biau2020}. 

\begin{thm}\label{thm:Decomposition}
Suppose that the conditions of Proposition \ref{prop:Dis_JS} hold and that $\left(\hat\varphi,\hat\xi\right)$ solve \eqref{eq:al_Lhat}. Then,
\[
d_{JS}(\mu,\mu_{\hat\varphi})\leq \varepsilon_{\textup{model}}+2\varepsilon_{\textup{sample}}(n),
\] 
where $\varepsilon_{\textup{model}}\coloneqq\inf_{\varphi\in\mathscr{H}^G} d_{JS}(\mu,\mu_\varphi)$ is the model error and 
\[
\varepsilon_{\textup{sample}}(n)\coloneqq\sup_{\varphi\in\mathscr{H}^G\atop \xi\in\mathscr{H}^D}\left|L(\varphi,\xi)-\hat L(\varphi,\xi)\right|
\]
is the sampling error at the sample size $n$.
\end{thm}
\begin{proof}
We first note that
\[
L\left(\varphi^{*},\xi_{\varphi^{*}}\right)=\min\limits _{\varphi\in\mathscr{H}_{K,\hat{K}}^{G}}\max\limits _{\xi\in\mathscr{H}_{B,C_{1},C_{2}}^{D}}L\left(\varphi,\xi\right)\leq\max\limits _{\xi\in\mathscr{H}_{B,C_{1},C_{2}}^{D}}L\left(\hat{\varphi}_{n},\xi\right)=L\left(\hat{\varphi}_{n},\xi_{\hat{\varphi}_{n}}\right),
\]
where the last equality follows from Proposition \ref{prop:Dis_JS}.
Next, in view of the definition of the sampling error, we have
\begin{align*}
L\left(\hat{\varphi}_{n},\xi_{\hat{\varphi}_{n}}\right) & \leq\hat{L}\left(\hat{\varphi}_{n},\xi_{\hat{\varphi}_{n}},n\right)+\varepsilon_{\textup{sample}}\left(n\right)\\
 & =\max\limits _{\xi\in\mathscr{H}^{D}}\hat{L}\left(\hat{\varphi}_{n},\xi,n\right)+\varepsilon_{\textup{sample}}\left(n\right)\\
 & =\min\limits _{\varphi\in\mathscr{H}^{G}}\max\limits _{\xi\in\mathscr{H}^{D}}\hat{L}\left(\varphi,\xi,n\right)+\varepsilon_{\textup{sample}}\left(n\right)\\
 & \leq\max\limits _{\xi\in\mathscr{H}^{D}}\hat{L}\left(\varphi^{*},\xi,n\right)+\varepsilon_{\textup{sample}}\left(n\right)\\
 & \leq\max\limits _{\xi\in\mathscr{H}^{D}}\left[L\left(\varphi^{*},\xi\right)+\varepsilon_{\textup{sample}}\left(n\right)\right]+\varepsilon_{\textup{sample}}\left(n\right)\\
 & =L\left(\varphi^{*},\xi_{\varphi^{*}}\right)+2\varepsilon_{\textup{sample}}\left(n\right),
\end{align*}
where $\varphi^{*}\in \mathscr{H}^G$ is arbitrary. Using the representation of the Jensen-Shannon divergence given in Proposition \ref{prop:Dis_JS}, we can subtract $\log(2)$ on both sides of the above inequality and take the infimum over $\varphi^{*}\in\mathscr{H}^G$ to conclude.
\end{proof}

Let us remark that due to the rather strong assumptions   on $\mathscr{H}^D$ in Proposition \ref{prop:Dis_JS}, we do not get model errors for discriminators.

Let us also remark that \cite{belomnestry2021} suggests a different estimate of the error, which avoids taking the supremum over $\mathscr{H}^G$ and $\mathscr{H}^D$ in the definition of $\varepsilon_{\textup{sample}}(n)$.   

In the following section, we give our assumptions on the density $f_\mu$ of the measure $\mu$ in terms of H\"older regularity and choose $\mathscr{H}^G$  such that the model error $\varepsilon_{\textup{model}}$ can be reduced to zero. That is, we show that in the hypothesis space there exist generators $\varphi$ that fulfill $\mu_\varphi=\mu$ exactly. We furthermore show that one can choose $\mathscr{H}^D$ such that the assumptions of Proposition \ref{prop:Dis_JS} and \ref{thm:Decomposition} are fulfilled. Then in Section \ref{sec:quantitative}, we consider the convergence $\varepsilon_{\textup{sample}}$ quantitatively.

\section{Generators, Discriminators and Consistency}
\label{Sec:Gener_and_Discr}
In this section, we first state our assumptions on the regularity of the density $f_\mu$ in terms of H\"{o}lder differentiability and we define suitable hypothesis spaces for generators and discriminators. Several technical results on H\"{o}lder functions are collected in Appendix \ref{app:A}, where we also prove an inverse function theorem for H\"{o}lder spaces as an important auxiliary result.

\subsection{Assumptions}
For a nonnegative
integer $k$ (or $k=\infty$), $C^{k}\left((0,1)^d,\mathbb{R}^{d_{2}}\right)$
 stands for the set of all $\mathbb{R}^{d_{2}}$-valued functions
with continuous $k$-th order derivatives and $C^{k}\left([0,1]^d,\mathbb{R}^{d_{2}}\right)$
is the set of all $\mathbb{R}^{d_{2}}$-valued functions
whose $k$-th derivatives have continuous extensions to $[0,1]^d$
(or, equivalently, the $k$-th derivatives are uniformly continuous
on $(0,1)^d$).  $C^{k}\left([0,1]^d,\mathbb{R}^{d_{2}}\right)$
is a Banach space with respect to the norm 
\[
\left\Vert f\right\Vert _{C^{k}\left([0,1]^d,\mathbb{R}^{d_{2}}\right)}\coloneqq\max_{\left|\beta\right|\leq k}\sup_{x\in [0,1]^d}\left|D_{\beta}f(x)\right|;
\]
here $\beta=(\beta_{1},\ldots,\beta_{d_{1}})\in\mathbb{Z}^{d_{1}}$
is a multi-index, $|\beta|\coloneqq\sum_{i=1}^{d_{1}}\beta_{i}$ is
its absolute value and $D_{\beta}f(x)\coloneqq\frac{\partial^{|\beta|}f(x)}{\partial x_{1}^{\beta_{1}}\cdots\partial x_{d_{1}}^{\beta_{d_{1}}}}$.
The Fr\'{e}chet derivative of $f$ at $x$  will be denoted by $Df(x)$.
As usual, the space of continuous functions $C^{0}$ and its norm
will be denoted by $C$ and $\Vert\cdot\Vert_{\infty}$, respectively.

The definition and many of the basic properties of H\"{o}lder spaces
of real-valued functions, given in \cite{Gil01}, extend without
difficulty to vector-valued ones:
\begin{defn}
Let $0<\alpha\leq1$,
and let $k$ be a nonnegative integer. The \emph{H\"{o}lder space}
$C^{k,\alpha}\left([0,1]^d,\mathbb{R}^{d_{2}}\right)$ consists
of all functions $f\in C^{k}\left([0,1]^d,\mathbb{R}^{d_{2}}\right)$
for which the following norm is finite: 
\[
\left\Vert f\right\Vert _{C^{k,\alpha}\left([0,1]^d,\mathbb{R}^{d_{2}}\right)}\coloneqq\left\Vert f\right\Vert _{C^{k}\left([0,1]^d,\mathbb{R}^{d_{2}}\right)}+\max_{\left|\beta\right|=k}\sup_{\substack{x,y\in (0,1)^d\\
x\ne y
}
}\frac{\left|D_{\beta}f(x)-D_{\beta}f(y)\right|}{\left|x-y\right|^{\alpha}}.
\]
\end{defn}

Our major requirement on the original measure is the following
\begin{assumption}
\label{assum01} $\mu$ is absolutely continuous with respect to $\lambda$,
and the Radon-Nikodym derivative $f_{\mu}\coloneqq\frac{d\mu}{d\lambda}$ 
is in $C^{k,\alpha}\left([0,1]^d\right)$ for some $\alpha\in(0,1]$
and $k\geq1$. Furthermore, $f_{\mu}$
satisfies the condition
\[
\kappa\coloneqq\min_{x\in[0,1]^d}f_{\mu}(x)>0.
\]
\end{assumption}

\begin{rem}\label{rem:Reg}
While the above assumptions seem restrictive, for any given $\mu$
one can easily define an approximate problem, such that Assumption
\ref{assum01} is fulfilled. To this avail, $[0,1]^{d}$ is embedded
in $\mathbb{R}^{d}$, convolved with $N(0,\kappa^{2}\mathbbm{1})$
and then reprojected to $[0,1]^{d}$ applying the $\operatorname{mod}\,1$-mapping
in all dimensions. On the side of the input data, this corresponds
to the mapping $Y\mapsto(Y+N)\,\operatorname{mod}\,1$, where $N\sim N(0,\kappa^{2}\mathbbm{1})$
is noise. This does not only allow to construct an explicit lower
bound for the density of the modified measure, but also an explicit
computation of the $C^{k,\alpha}$-norm.  
\end{rem}

\subsection{The hypothesis space of generators}

 Our goal is to define suitable hypothesis spaces for generators such that $\mu=f_\mu\lambda$ is realizable under the Assumption \ref{assum01}. On the other hand, the hypothesis space should not be too large, in order to keep control on $\varepsilon_{\textup{sample}}(n)$, which is the same as avoidance of overfitting. 

\begin{defn}
\label{def:set_of_gener} Let $K$ and $\hat{K}$ be positive constants.
The set of all $k\geq1$ times $\alpha$-H\"{o}lder differentiable
bijective functions $\varphi:[0,1]^{d}\to[0,1]^{d}$, satisfying the conditions
$\left\Vert \varphi\right\Vert _{C^{k,\alpha}}\leq K$ and $\left\Vert \varphi^{-1}\right\Vert _{C^{k,\alpha}}\leq\hat{K}$,
will be denoted by $\mathscr{H}_{K,\hat{K}}^{G}$. 
\end{defn}
The spaces $\mathscr{H}_{K,\hat{K}}^{G}$ are candidates for generator spaces, provided the constants $K$ and $\hat K$ are sufficiently large. Before doing so, we convince ourselves that the requirements of Proposition \ref{prop:Dis_JS} concerning the existence of densities are fulfilled. Also, we have a closer look at the regularity of the densities in terms of their H\"{o}lder norms. 
\begin{prop}
\label{prop:gen_dens}$\mathscr{H}_{K,\hat{K}}^{G}$ is a closed subset
of the space $C^{k,\alpha}\left([0,1]^{d},\mathbb{R}^d\right)$. For each $\varphi\in\mathscr{H}_{K,\hat{K}}^{G}$,
the pushforward measure $\mu_{\varphi}=\varphi_{*}\lambda$
is absolutely continuous with respect to $\lambda$, and its
density $f_{\varphi}$ is a $\left(k-1\right)$-times $\alpha$-H\"{o}lder
differentiable function on $[0,1]^{d}$.
The density $f_{\varphi}$ is given by the formula
\begin{equation}
f_{\varphi}=\left|J_{\varphi^{-1}}\right|\label{eq:gen_density}
\end{equation}
where $J_{\varphi^{-1}}$ is the Jacobian determinant of $\varphi^{-1}$. The following estimate bounds $f_\varphi$ in terms of $K$ and $\hat{K}$ from above and below
\begin{equation}
\frac{1}{d!K^{d}}\leq f_{\varphi}\leq d!\hat{K}^{d}.\label{eq:densit_bounds}
\end{equation}
Moreover,
\[
\left\Vert f_{\varphi}\right\Vert _{C^{k-1,\alpha}}\leq2d!d^{k+1}\hat{K}^{d}.
\]
\end{prop}

\begin{proof}
Using \eqref{eq:det_up_bound}, for every $\varphi\in\mathscr{H}_{K,\hat{K}}^{G}$,
we may estimate
\begin{equation}
\left|J_{\varphi}\right|\leq d!\max_{x\in[0,1]^{d}}\left[\max_{1\leq i,j\leq d}\left|\frac{\partial\varphi_{i}(x)}{\partial x_{j}}\right|\right]^{d}\leq d!K^{d},\label{eq:Jacob_bound_phi}
\end{equation}
and, similarly,
\begin{equation}
\left|J_{\varphi^{-1}}\right|\leq d!\hat{K}^{d}.\label{eq:Jacob_bound_phi_inv}
\end{equation}
The latter, in view of the equality $J_{\varphi}=\frac{1}{J_{\varphi^{-1}}\circ\varphi}$,
gives $\left|J_{\varphi}\right|\geq\tfrac{1}{d!\hat{K}^{d}}$. Therefore
the statement (b) of Theorem \ref{Th: hoeld_diffeomor} shows that
$\mathscr{H}_{K,\hat{K}}^{G}$ is closed in $C^{k,\alpha}\left([0,1]^{d},\mathbb{R}^d\right)$.

The change-of-variables formula shows that
\[
\varphi_{*}\lambda\left(A\right)=\lambda\left(\varphi^{-1}\left(A\right)\right)=\int\limits _{\varphi^{-1}\left(A\right)}d\lambda=\int\limits _{A}\left|J_{\varphi^{-1}}\right|d\lambda,
\]
for any $\varphi\in\mathscr{H}_{K,\hat{K}}^{G}$ and $A\in\mathcal{\mathscr{B}}\left([0,1]^{d}\right)$;
hence each $\varphi\in\mathscr{H}_{K,\hat{K}}^{G}$ generates a measure
$\mu_{\varphi}=\varphi_{*}\lambda$ with a density $f_{\varphi}=\left|J_{\varphi^{-1}}\right|$.
Equations \eqref{eq:Jacob_bound_phi} and \eqref{eq:Jacob_bound_phi_inv}, together
with the equality $J_{\varphi^{-1}}=\frac{1}{J_{\varphi}\circ\varphi^{-1}}$,
imply the estimate \eqref{eq:densit_bounds}.

In view of Corollary \ref{cor:hoeld_norm_determinant},
\[
\left\Vert J_{\varphi^{-1}}\right\Vert _{C^{k-1,\alpha}}\leq2d!d^{k}\max\bigl\{1,[\operatorname{diam}([0,1]^{d})]^{1-\alpha}\bigr\}\left\Vert \varphi^{-1}\right\Vert _{C^{k,\alpha}}^{d}\leq2d!d^{k+1}\hat{K}^{d},
\]
which completes the proof.
\end{proof}

Compactness is an important property for hypothesis spaces as it plays a central role in the uniform law of large numbers (see e.g.\ \cite{Ferguson.2017}) that is used to control $\varepsilon_{\textup{sample}}(n)$ in the limit $n\to\infty$. However, in view of Stone's theorem, we cannot expect compactness for the space $\mathscr{H}_{K,\hat{K}}^{G}$ in the $C^{k,\alpha}$-topology. Therefore it is crucial that the hypothesis space of generators is compact with respect to a topology that is just a little bit weaker: 

\begin{prop}
\label{prop:compact_gener}$\mathscr{H}_{K,\hat{K}}^{G}$ is compact
in $C^{k,\alpha'}\left([0,1]^{d},\mathbb{R}^d\right)$ for $0<\alpha'<\alpha$.
\end{prop}

\begin{proof}
Indeed, $\mathscr{H}_{K,\hat{K}}^{G}$ is a bounded subset of the
space $C^{k,\alpha}\left([0,1]^{d},\mathbb{R}^d\right)$. The embedding of  $C^{k,\alpha}\left([0,1]^{d},\mathbb{R}^d\right)$  into $C^{k,\alpha'}\left([0,1]^{d},\mathbb{R}^d\right)$
is compact \cite[Lemma 6.33]{Gil01}, hence $\mathscr{H}_{K,\hat{K}}^{G}$ is relatively compact
in $C^{k,\alpha'}\left([0,1]^{d}\right)$. Obviously, the uniform bound $K$ with respect to the $\|\cdot\|_{C^{k,\alpha}}$-norm is stable under $C^{k,\alpha'}$ limits $\varphi_n\to\varphi\in C^{k,\alpha'}([0,1]^d,\mathbb{R}^d)$. Using the inverse function theorem for H\"{o}lder spaces (Theorem \ref{Th: hoeld_diffeomor} (b)), we see that $\|\varphi^{-1}\|_{C^{k,\alpha}}=\lim_n\|\varphi_n^{-1}\|_{C^{k,\alpha}}\leq \hat{K}$. Hence, $\varphi\in \mathscr{H}_{K,\hat{K}}^{G}$
and thus this space is closed in $C^{k,\alpha'}\left([0,1]^{d},\mathbb{R}^d\right)$ and therefore is compact.
\end{proof}

In the next step, we prove the existence of a generator for $\mu$ and compute H\"{o}lder norms of the generator and its inverse. Our construction is based on the Rosenblatt transformation \cite{ros52}, which recursively computes a transformation based on conditional cumulative distribution functions and their inverse, i.e.\ conditional quantile functions. 

Recall the definition of the conditional density (see \cite[Sec. 2.7]{Shir16}).
Let $Y_{j}:[0,1]^{d}\to\mathbb{R}^{d_{j}}\;\,\left(j=1,2\right)$ be
random variables such that the pair $Y\coloneqq\left(Y_{1},Y_{2}\right)$
has a density $f_{Y}\left(y_{1},y_{2}\right)$. Then $Y_{2}$ has
a marginal density given by the formula $f_{Y_{2}}\left(y_{2}\right)=\int_{\mathbb{R}^{d_{1}}}f_{Y}\left(y_{1},y_{2}\right)dy_{1}\;\,\left(y_{j}\in\mathbb{R}^{d_{j}};j=1,2\right)$.
We define the conditional density $f_{Y_{1}|Y_{2}}\left(\cdot|y_{2}\right)=f_{Y_{1}|Y_{2}=y_{2}}\left(\cdot\right)$
by the formula
\begin{equation}
f_{Y_{1}|Y_{2}}\left(y_{1}|y_{2}\right)=\left\{ \begin{gathered}\frac{f_{Y}\left(y_{1},y_{2}\right)}{f_{Y_{2}}\left(y_{2}\right)}\quad\mbox{if}\quad f_{Y_{2}}\left(y_{2}\right)>0,\hfill\\
0\quad\mbox{if}\quad f_{Y_{2}}\left(y_{2}\right)=0.\hfill
\end{gathered}
\right.\label{eq:cond_dens}
\end{equation}
The conditional (cumulative) distribution function $F_{Y_{1}|Y_{2}}\left(\cdot|y_{2}\right)=F_{Y_{1}|Y_{2}=y_{2}}\left(\cdot\right)$
is given by \footnote{For $a_{j}=\left(a_{j}^{1},\ldots,a_{j}^{d}\right)\;\,\left(j=1,2\right)$
we put $\int\limits _{a_{1}}^{a_{2}}=\int\limits _{a_{1}^{1}}^{a_{2}^{1}}\cdots\int\limits _{a_{1}^{d}}^{a_{2}^{d}}$
for brevity.} 
\[
F_{Y_{1}|Y_{2}}\left(y_{1}|y_{2}\right)=\int\limits _{-\infty}^{y_{1}}f_{Y_{1}|Y_{2}}\left(s|y_{2}\right)ds\quad\left(y_{1}\in\mathbb{R}^{d_{1}}\right).
\]

After these preparations, we first find the inverse of the generator that generates samples from $\mu$. Proving H\"{o}lder differentiability for the inverse along with a lower bound for the Jacobian determinant enables us to apply the inverse function theorem for H\"{o}lder functions, which we prove in the appendix (see Theorem \ref{Th: hoeld_diffeomor}). For now, let $Y=\left(Y_{1},\dots,Y_{d}\right)$ be a random vector with distribution
$\mu$ satisfying Assumption \ref{assum01}. 
Consider the \emph{Rosenblatt
transformation} $\psi$, given by the formula
\[
\psi\left(y_{1},y_{2},\ldots,y_{d}\right) =\left(F_{Y_{1}}\left(y_{1}\right),F_{Y_{2}|Y_{1}}\left(y_{2}|y_{1}\right),\ldots,F_{Y_{d}|\left(Y_{1},\dots,Y_{d-1}\right)}\left(y_{d}\,|\negthinspace\left(y_{1},\dots,y_{d-1}\right)\right)\right)
\]
for every $\left(y_{1},\ldots,y_{d}\right)\in[0,1]^{d}$ (see \cite{ros52}).
The inequalities $0\leq F_{Y_{1}}\leq1$ and $0\leq F_{Y_{j}|\left(Y_{1},\dots,Y_{j-1}\right)}\leq1\;\,\left(2\leq j\leq d\right)$
show that $\psi\left([0,1]^{d}\right)\subset[0,1]^{d}$.

The following proposition proves that $\mu$ is \emph{realizable} in $\mathscr{H}_{K,\hat{K}}^{G}$:
\begin{prop}
\label{Th: Ros_Hoelder-diff_ty} Under Assumption \ref{assum01},
the Rosenblatt transformation $\psi$ has the following properties:

\begin{enumerate}
\item[(i)] $\psi$ is $k$-times $\alpha$-H\"{o}lder differentiable, i.e.,
$\psi\in C^{k,\alpha}\left([0,1]^{d}\right)$.
\item[(ii)] $\psi$ is bijective and the inverse map $\psi^{-1}$ is given by
the formula
\[
\psi^{-1}(x_{1},x_{2},\ldots,x_{d})=\left(y_{1},y_{2},\ldots,y_{d}\right)
\]
where\footnote{For a non-injective distribution function $F$ we put $F^{-1}(y)\coloneqq\inf\left\{ x\in\mathbb{R}:F(x)\geq y\right\} $.}
\[
\arraycolsep=0.1em\begin{array}{llc}
y_{1} & =F_{Y_{1}}^{-1}(x_{1}),\\
y_{2} & =F_{Y_{2}|Y_{1}}^{-1}(x_{2}|y_{1}),\\
\hdotsfor[5.5]{2}\\
y_{d} & =F_{Y_{d}|(Y_{1},\dots,Y_{d-1})}^{-1}(x_{d}|(y_{1},\ldots, & y_{d-1})).
\end{array}
\]
\item[(iii)] The Jacobian determinant of $\psi$ coincides with the density $f_{\mu}$.
\end{enumerate}
\end{prop}
\begin{proof}
In view of Assumption \ref{assum01}, Proposition \ref{prop:hoel_norm_prod}
and Lemma \ref{lem:hoel_norm_reciprocal}, the densities
\[
f_{Y_{1}}\left(\cdot\right),f_{Y_{2}|Y_{1}}\left(\cdot|y_{1}\right),\ldots,f_{Y_{d}|\left(Y_{1},\dots,Y_{d-1}\right)}\left(\cdot\,|\negthinspace\left(y_{1},\dots,y_{d-1}\right)\right)
\]
are $k$-times H\"{o}lder differentiable. Clearly, the integration
preserves the H\"{o}lder continuity. Therefore the conditional cumulative density functions
\[
F_{Y_{1}}\left(\cdot\right),F_{Y_{2}|Y_{1}}\left(\cdot|y_{1}\right),\ldots,F_{Y_{d}|\left(Y_{1},\dots,Y_{d-1}\right)}\left(\cdot\,|\negthinspace\left(y_{1},\dots,y_{d-1}\right)\right)
\]
are $k$-times H\"{o}lder differentiable, too. Thus, $\psi\in C^{k,\alpha}\left([0,1]^{d},\mathbb{R}^d\right)$.

In order to prove the bijectivity of $\psi$, we choose an arbitrary
point $x=(x_{1},\ldots,x_{d})\in[0,1]^{d}$ and show that the equation
$\psi\left(y\right)=x$ has a unique solution. Put $\psi_{1}\left(y_{1}\right)\coloneqq F_{Y_{1}}\left(y_{1}\right)\;\,\left(y_{1}\in[0,1]\right)$
and
\[
\psi_{j}\left(y_{1},\dots,y_{j}\right)\coloneqq F_{Y_{j}|\left(Y_{1},\dots,Y_{j-1}\right)}\left(y_{j}|\left(y_{1},\ldots,y_{j-1}\right)\right)\quad\left(y_{1},\ldots,y_{j}\in[0,1]\right),
\]
for $2\leq j\leq d$. We rewrite the definition of $\psi$ in the
form
\[
\psi\left(y_{1},y_{2},\ldots,y_{d}\right)=\left(\psi_{1}\left(y_{1}\right),\ldots,\psi_{d}\left(y_{1},\dots,y_{d}\right)\right).
\]
Assumption \ref{assum01} ensures that
\begin{align}
\frac{\partial\psi_{1}(y_{1})}{\partial y_{1}} & =f_{Y_{1}}\left(y_{1}\right)>0\quad\left(y_{1}\in(0,1)\right),\label{eq:psi_deriv_1}\\
\frac{\partial\psi_{j}(y_{1},\dots,y_{j})}{\partial y_{j}} & =f_{Y_{j}|\left(Y_{1},\dots,Y_{j-1}\right)}\left(y_{j}|(y_{1},\ldots,y_{j-1}\right))>0,\label{eq:psi_deriv_2}\\
 & \left(y_{1},\ldots,y_{j}\in(0,1);2\leq j\leq d\right)\nonumber 
\end{align}
hence $\psi_{j}\left(y_{1},\dots,y_{j}\right)$ is strictly increasing
with respect to $y_{j}$ on the segment $[0,1]$, for $1\leq j\leq d$.
Moreover, $\psi_{j}|_{y_{j}=0}=0,\:\psi_{j}|_{y_{j}=1}=1$. Therefore
the equation $\psi_{1}\left(y_{1}\right)=x_{1}$ has a unique solution
which we denote by $b_{1}$. Next, the equation $\psi_{2}\left(b_{1},y_{2}\right)=x_{2}$
has a unique solution which we denote by $b_{2}$, etc. Thus, we see
in a finite number of steps that the equation $\psi\left(y\right)=x$
has a unique solution $y=\left(b_{1},b_{2},\ldots,b_{d}\right)$ where
\[
b_{1}=F_{Y_{1}}^{-1}(x_{1}),b_{2}=F_{Y_{2}|Y_{1}}^{-1}(x_{2}|b_{1}),\ldots,b_{d}=F_{Y_{d}|\left(Y_{1},\dots,Y_{d-1}\right)}^{-1}\left(x_{d}|(b_{1},\ldots,b_{d-1})\right).
\]
To compute the Jacobian determinant $J_{\psi}$, observe that $\frac{\partial\psi_{i}}{\partial y_{j}}=0\;\:\left(1\leq i<j\leq d\right)$,
i.e., the Jacobian matrix of $\psi$ is lower triangular. Hence
\[
J_{\psi}=\frac{\partial\psi_{1}}{\partial y_{1}}\cdot\frac{\partial\psi_{2}}{\partial y_{2}}\cdots\frac{\partial\psi_{d}}{\partial y_{d}}.
\]
\eqref{eq:cond_dens} and \eqref{eq:psi_deriv_2} give
\[
\frac{\partial\psi_{j}\left(y_{1},\dots,y_{j}\right)}{\partial y_{j}}=\frac{f_{\left(Y_{1},\dots,Y_{j}\right)}\left(y_{1},\ldots,y_{j}\right)}{f_{\left(Y_{1},\dots,Y_{j-1}\right)}\left(y_{1},\ldots,y_{j-1}\right)}\quad\left(y_{1},\ldots,y_{j}\in\left(0,1\right);\,2\leq j\leq d\right),
\]
which, together with \eqref{eq:psi_deriv_1}, yields
\[
J_{\psi}\left(y_{1},\ldots,y_{d}\right) = f_{Y_{1}}\left(y_{1}\right)\cdot\frac{f_{\left(Y_{1},Y_{2}\right)}\left(y_{1},y_{2}\right)}{f_{Y_{1}}\left(y_{1}\right)}\cdots\frac{f_{\left(Y_{1},\dots,Y_{d}\right)}\left(y_{1},\ldots,y_{d}\right)}{f_{\left(Y_{1},\dots,Y_{d-1}\right)}\left(y_{1},\ldots,y_{d-1}\right)} = f_{\left(Y_{1},\dots,Y_{d}\right)}\left(y_{1},\ldots,y_{d}\right).
\]
The proof is complete.
\end{proof}

The inverse of the Rosenblatt transformation $\psi$ will be denoted
by $\phi$. Whenever clarity requires, we will write $\phi_{\mu}$
instead of $\phi$. The second statement of Theorem \ref{Th: Ros_Hoelder-diff_ty}
shows that $\phi$ can be written in the form
\[
\phi(x_{1},\ldots,x_{d})=\left(\phi_{1}(x_{1}),\ldots,\phi_{d}(x_{1},\dots,x_{d})\right),
\]
hence $\phi$, like $\psi$, has a lower triangular Jacobi matrix.

The following simple statement shows that for sufficiently large $K$
and $\hat{K}$ the inverse Rosenblatt transformation $\phi$ is a
generator for the measure $\mu$. Note that the procedure described in Remark \ref{rem:Reg} actually allows to compute suitable $K$ and $\hat K$ from the parameters of the regularization. 
\begin{thm}
\label{prop: Ros_is_generator} If $\mu$ satisfies Assumption \ref{assum01},
then there exists a generator $\phi\in \mathscr{H}_{K,\hat{K}}^G$ such that $\mu=\phi_{*}\lambda$
holds, provided $K$ and $\hat{K}$ are sufficiently large. In this case, the model error in Theorem \ref{thm:Decomposition} vanishes, i.e.\ $\varepsilon_\textup{model}=0$.
\end{thm}

\begin{proof}
Theorems \ref{Th: Ros_Hoelder-diff_ty} and \ref{Th: hoeld_diffeomor}
show that $\phi=\psi^{-1}\in C^{k,\alpha}\left([0,1]^{d},[0,1]^d\right)$. Next,
for any $A\in\mathcal{\mathscr{B}}\left([0,1]^{d}\right)$, the change-of-variables
theorem, together with the statement (iii) of Theorem \ref{Th: Ros_Hoelder-diff_ty},
gives 
\[
\mu\left(A\right)=\int\limits _{A}f_{\mu}d\lambda=\int\limits _{A}\left|J_{\psi}\right|d\lambda=\int\limits _{\psi\left(A\right)}d\lambda=\lambda\left(\phi^{-1}\left(A\right)\right).\qedhere
\]
\end{proof}
\begin{rem}
Related results on the existence of generators can also be obtained using optimal transport, see e.g.\ \cite{villani2009optimal}, or rearrangement techniques \cite{brenier1991polar}. It might be of interest that our rather elementary analysis of the Rosenblatt transformation leads to comparable results on the existence and regularity of the generator. 

The special, lower triangular shape of the inverse Rosenblatt transformation considerably simplifies the numerical computation of $f_\varphi=|J_{\varphi^{-1}}|$, which  would be of interest for a likelihood based (non adversarial) training of the generators, as it is done for normalizing flows \cite{dinh2014nice,kobyzev2020normalizing}, in particular in the context of inevitable neural networks \cite{teshima2020coupling}.  Also, our findings could be of interest in the theory of copulae \cite{joe2014dependence}.      
\end{rem}

\subsection{The hypothesis space of discriminators}

\label{Subsc_discrim} Here we present the hypothesis space of discriminators $\mathscr{H}_{B,C_{1},C_{2}}^{D}$
which is affiliated to the set of generators in the sense that if
the data generating measure $\mu$ is realizable in $\mathscr{H}_{K,\hat{K}}^{G}$,
then the optimal discriminator that separates data from $\phi_{*}\lambda$
and $\mu$, is realizable in $\mathscr{H}_{B,C_{1},C_{2}}^{D}$, see
Theorem \ref{thm:convenient} below. We therefore show that in our setting the assumptions of Theorem \ref{thm:Decomposition} hold and we can bound the Jensen Shannon divergence of the true measure with the sampling error.

\begin{defn}
\label{def:set_of_discrim} Let $B\in\left(0,\frac{1}{2}\right)$
and $C_{i}>0\;\,\,(i=1,2)$. The collection of all functions $\xi\in C^{k-1,\alpha}\left([0,1]^{d},\mathbb{R}\right)$,
satisfying the conditions $B\leq\xi\leq1-B$, $\left\Vert D\xi\right\Vert _{\infty}\leq C_{1}$
and $\left\Vert \xi\right\Vert _{C^{k-1,\alpha}\left([0,1]^{d},\mathbb{R}\right)}\leq C_{2}$,
is called the hypothesis space of discriminators and is denoted
by $\mathscr{H}_{B,C_{1},C_{2}}^{D}$.
\end{defn}

Again, to achieve control of the sample error, compactness plays a crucial role. Therefore we establish it in the following proposition:

\begin{prop}
\label{prop:compact_discr}$\mathscr{H}_{B,C_{1},C_{2}}^{D}$ is compact
in $C^{k-1,\alpha'}\left([0,1]^{d},\mathbb{R}\right)$ for $0<\alpha'<\alpha$.
\end{prop}
\begin{proof}
The argument provided in the proof of Proposition \ref{prop:compact_gener} carries over to this much simpler case.
\end{proof}

In the next step, we compute constants $B$, $C_1$ and $C_2$ as functions of $K$ and $\hat{K}$, which guarantees that the assumptions of Proposition \ref{prop:Dis_JS} and hence those of Theorem \ref{thm:Decomposition} with respect to the generators are fulfilled. The following lemma, which heavily relies on our analysis of H\"{o}lder functions in Appendix \ref{app:A},  provides these technical results:

\begin{lem}
\label{lem:opt_discriminator} Let $\varphi,\varphi'\in\mathscr{H}_{K,\hat{K}}^{G}$,
and let $\xi_{\varphi,\varphi'}:=\tfrac{f_{\varphi}}{f_{\varphi}+f_{\varphi'}}$.
If $K$ and $\hat{K}$ satisfy the conditions $4d!d^{k+1}\hat{K}^{d}>1$
and $d!K^{d}>2$, then $\xi_{\varphi,\varphi'}\in\mathscr{H}_{B,C_{1},C_{2}}^{D}$
with $B=\text{\ensuremath{\bigl[}1+(d!\ensuremath{)^{2}K^{d}\hat{K}^{d}\,\bigr]}}^{-1}$,\linebreak
$C_{1}=2d^{3}\left(d!\right)^{4}K^{2d}\hat{K}^{2d+1}$ and $C_{2}=\bigl[4(d!)^{2}d^{k+2}K^{d}\hat{K}^{d}\bigr]^{2^{k+4}}$.
\end{lem}

\begin{proof}
Proposition \ref{prop:gen_dens} and Lemma \ref{lem:hoel_norm_reciprocal}
show that $\xi_{\varphi,\varphi'}\in C^{k-1,\alpha}\left([0,1]^{d},\mathbb{R}\right)$.
Since $f_{\varphi}$ and $f_{\varphi'}$ admit the same estimate \eqref{eq:densit_bounds},
hence
\[
\frac{1}{(d!)^{2}K^{d}\hat{K}^{d}}\leq\frac{f_{\varphi'}}{f_{\varphi}}\leq(d!)^{2}K^{d}\hat{K}^{d}.
\]
Using this for $\xi_{\varphi,\varphi'}=\left(1+\frac{f_{\varphi'}}{f_{\varphi}}\right)^{-1}\negthickspace$,
we obtain $B\leq\xi_{\varphi,\varphi'}\leq1-B$ with
\[
B\coloneqq\frac{1}{1+(d!)^{2}K^{d}\hat{K}^{d}}.
\]

To obtain an upper bound for $\left\Vert D\left(\xi_{\varphi,\varphi'}\right)\right\Vert _{\infty}$,
we will first find a Lipschitz estimate for the function $\xi_{\varphi,\varphi'}$.
We have
\[
\xi_{\varphi,\varphi'}(x_{1})-\xi_{\varphi,\varphi'}(x_{2})=\frac{f_{\varphi'}(x_{2})\left[f_{\varphi}(x_{1})-f_{\varphi}(x_{2})\right]-f_{\varphi}(x_{2})\left[f_{\varphi'}(x_{1})-f_{\varphi'}(x_{2})\right]}{\left[f_{\varphi}(x_{1})+f_{\varphi'}(x_{1})\right]\left[f_{\varphi}(x_{2})+f_{\varphi'}(x_{2})\right]}.
\]
The estimate \eqref{eq:densit_bounds} gives
\[
\tfrac{f_{\varphi}(x_{2})}{\left[f_{\varphi}(x_{1})+f_{\varphi'}(x_{1})\right]\left[f_{\varphi}(x_{2})+f_{\varphi'}(x_{2})\right]}\leq d!K^{d},\quad\tfrac{f_{\varphi'}(x_{2})}{\left[f_{\varphi}(x_{1})+f_{\varphi'}(x_{1})\right]\left[f_{\varphi}(x_{2})+f_{\varphi'}(x_{2})\right]}\leq d!K^{d},
\]
which, together with \eqref{eq:gen_density} and \eqref{eq:densit_bounds},
yields
\begin{align*}
\left|\xi_{\varphi,\varphi'}(x_{1})-\xi_{\varphi,\varphi'}(x_{2})\right| & \leq d!K^{d}\sum_{h\in\left\{ \varphi,\varphi'\right\} }f_{h}(x_{1})f_{h}(x_{2})\left|\frac{1}{f_{h}(x_{1})}-\frac{1}{f_{h}(x_{2})}\right|\\
 & \leq(d!)^{3}K^{d}\hat{K}^{2d}\sum_{h\in\left\{ \varphi,\varphi'\right\} }\left|J_{h}\left(h^{-1}(x_{1})\right)-J_{h}\left(h^{-1}(x_{2})\right)\right|.
\end{align*}
By induction on $d$, one can easily prove the inequality
\begin{equation}
\left|a_{1}\cdots a_{d}-b_{1}\cdots b_{d}\right|\leq d\left[\max_{1\leq j\leq d}\left\{ \left|a_{j}\right|,\left|b_{j}\right|\right\} \right]^{d-1}\max_{1\leq j\leq d}\left|a_{j}-b_{j}\right|\label{eq:prod_change}
\end{equation}
where $a_{j},b_{j}\;\,\left(1\leq j\leq d\right)$ are arbitrary real
numbers. Using this inequality and the definition of a determinant,
we can easily obtain the estimate
\[
\left|\det\left(a_{ij}\right)_{d\times d}-\det\left(b_{ij}\right)_{d\times d}\right|\leq d\cdot d!\left[\max_{1\leq i,j\leq d}\left\{ \left|a_{ij}\right|,\left|b_{ij}\right|\right\} \right]^{d-1}\max_{1\leq i,j\leq d}\left|a_{ij}-b_{ij}\right|,
\]
for any matrices $\left(a_{ij}\right)_{d\times d},\left(b_{ij}\right)_{d\times d}\in M_{dd}\left(\mathbb{R}\right)$.
Applying this estimate to the Jacobian matrices $J_{h}\left(y_{1}\right)$
and $J_{h}\left(y_{2}\right)$ yields
\[
\left|J_{h}\left(y_{1}\right)-J_{h}\left(y_{2}\right)\right|\leq d\cdot d!K^{d-1}\left|\left(Dh\right)\left(y_{1}\right)-\left(Dh\right)\left(y_{2}\right)\right|\quad\left(y_{1},y_{2}\in[0,1]^{d}\right).
\]
The latter, combined with the estimate $\left\Vert Dh\right\Vert \leq d\left\Vert h\right\Vert _{C^{1}}\;\:\left(h\in C^{1}([0,1]^{d})\right)$
and with the mean value theorem for vector-valued functions, gives
\[
\left|J_{h}\left(y_{1}\right)-J_{h}\left(y_{2}\right)\right|\leq d\cdot d!K^{d-1}\cdot dK\left|y_{1}-y_{2}\right|\quad\left(y_{1},y_{2}\in[0,1]^{d}\right),
\]
therefore
\begin{align*}
\left|\xi_{\varphi,\varphi'}(x_{1})-\xi_{\varphi,\varphi'}(x_{2})\right| & \leq d^{2}\left(d!\right)^{4}K^{2d}\hat{K}^{2d}\sum_{h\in\left\{ \varphi,\varphi'\right\} }\left|h^{-1}(x_{1})-h^{-1}(x_{2})\right|\\
 & \leq2d^{3}\left(d!\right)^{4}K^{2d}\hat{K}^{2d+1}\left|x_{1}-x_{2}\right|\quad\left(x_{1},x_{2}\in[0,1]^{d}\right).
\end{align*}
This inequality implies that $\left\Vert D\xi_{\varphi,\varphi'}(x)\right\Vert \leq2d^{3}\left(d!\right)^{4}K^{2d}\hat{K}^{2d+1}\;\,\left(x\in(0,1)^{d}\right)$.
Indeed, choose an arbitrary $x\in(0,1)^{d}$, then $x\in(t,1-t)^{d}$
for all sufficiently small $t>0$, and we have
\begin{align*}
\left\Vert D\left(\xi_{\varphi,\varphi'}\right)(x)\right\Vert = & \sup_{|h|=1}\frac{\left|D\left(\xi_{\varphi,\varphi'}\right)(x)ht\right|}{\left|ht\right|}\leq\sup_{|h|=1}\frac{\left|\xi_{\varphi,\varphi'}\left(x+ht\right)-\xi_{\varphi,\varphi'}(x)\right|}{\left|ht\right|}\\
 & +\sup_{|h|=1}\frac{\left|D\left(\xi_{\varphi,\varphi'}\right)(x)ht-\left[\xi_{\varphi,\varphi'}\left(x+ht\right)-\xi_{\varphi,\varphi'}(x)\right]\right|}{\left|ht\right|}\\
\leq & \sup_{|h|=1}\frac{\left|D\left(\xi_{\varphi,\varphi'}\right)(x)ht-\left[\xi_{\varphi,\varphi'}\left(x+ht\right)-\xi_{\varphi,\varphi'}(x)\right]\right|}{\left|ht\right|}\\
 & +2d^{3}\left(d!\right)^{4}K^{2d}\hat{K}^{2d+1}.
\end{align*}
Letting $t\to0$, we obtain the desired inequality.

In view of Proposition \ref{prop:gen_dens}, $\left\Vert f_{\varphi}\right\Vert _{C^{k-1,\alpha}}\leq2d!d^{k+1}\hat{K}^{d}$,
$\left\Vert f_{\varphi}+f_{\varphi'}\right\Vert _{C^{k-1,\alpha}}\leq4d!d^{k+1}\hat{K}^{d}$
and $f_{\varphi}+f_{\varphi'}\geq\tfrac{2}{d!K^{d}}$. Since $4d!d^{k+1}\hat{K}^{d}>1$
and $\tfrac{d!K^{d}}{2}>1$, we may apply Proposition \ref{prop:Hoel_norm_frac}
to obtain
\[
\left\Vert \xi_{\varphi,\varphi'}\right\Vert _{C^{k-1,\alpha}}\leq\bigl[4(d!)^{2}d^{k+1}\hat{K}^{d}K^{d}\bigr]^{2^{k+4}}\max\bigl\{1,[\operatorname{diam}([0,1]^{d})]^{2(1-\alpha)}\bigr\}\leq C_{2}.
\]
which completes the proof.
\end{proof}

We can now wrap up our results in the following theorem, which shows that we can drop model errors with regard to generators and discriminators:

\begin{thm}
\label{thm:convenient}
Suppose the Assumptions \ref{assum01} hold and the constants $K$ and $\hat{K}$ are sufficiently large. Furthermore, let $B$, $C_1$ and $C_2$ be given as in Lemma \ref{lem:opt_discriminator}. Then the assumptions of Theorem \ref{thm:Decomposition} hold and 
\[
d_{JS}(\mu,\mu_{\hat{\varphi}})\leq 2\varepsilon_{\textup{sample}}(n),
\]
where $\hat\varphi$ solves the minimax problem \eqref{eq:al_Lhat}.
\end{thm}
\begin{proof}
Combine the statements of \ref{thm:Decomposition}, Theorem \ref{prop: Ros_is_generator} and Lemma \ref{lem:opt_discriminator}.
\end{proof}

\subsection{Consistency of generative adversarial learning}

From the construction of discriminators we have seen that the hypothesis spaces $\mathscr{H}^G_{K,\hat{K}}$ and $\mathscr{H}_{B,C_1,C_2}$ are large enough to contain optimal generators and discriminators, respectively, provided the constants $K,\hat{K},B,C_1$ and $C_2$ are properly chosen. Therefore we conclude the consistency of generative adversarial learning by a simple application of the uniform law of large numbers:

\begin{thm}
\label{thm:consistency}
Suppose that $k\geq 1$, $K$ nd $\hat{K}$ are sufficiently large and let $B$, $C_1$ and $C_2$ be given as in Lemma \ref{lem:opt_discriminator} and let $\mathscr{H}^G=\mathscr{H}^G_{K,\hat{K}}$ and $\mathscr{H}^D=\mathscr{H}^D_{B,C_1,C_2}$. Let $(\hat \varphi_n,\hat \xi_n)$ solve the minimax Problem \ref{eq:al_Lhat} for sample size $n$. Then
\[
d_{JS}(\mu,\mu_{\hat{\varphi}_n})\to 0
\]
holds almost surely, i.e.\ generative adversarial learning is consistent with respect to the Jensen-Shannon divergence.
\end{thm} 

\begin{proof}
By theorem \ref{thm:convenient}, it is enough to prove $\varepsilon_{\textup{sample}}(n)\to 0$ almost surely. By the Propositions \ref{prop:compact_gener} and \ref{prop:compact_discr}, for $0<\alpha'<\alpha$ the hypothesis spaces $\mathscr{H}^G$ and $\mathscr{H}^D$ are compact in the $C^{k,\alpha'}$ and $C^{k-1,\alpha'}$-topology , and thus their direct product $\mathscr{H}^G\times \mathscr{H}^D$ is compact as well.

 Furthermore, we note that the single summands  $L(\varphi,\xi)-\frac{1}{2} [\log(\xi(Y_i))+\log(1-\xi(\phi(Z_i))$ in the definition of $\varepsilon_{\textup{sample}}(n)$ are independent and identically distributed. 

In addition, they and uniformly bounded  by $2\max\{-log(B),-\log(1-B)\}$. Furthermore it is easy to see that these expressions are continuous in $(\varphi,\xi)$ with respect to the $C^{k,\alpha'}\times C^{k-1,\alpha'}$-topology. Hence, the conditions of the uniform law of large numbers, see e.g.\ \cite{Ferguson.2017}, hold and the claim  $\varepsilon_\textup{sample}(n)\to 0$ (a.s.) follows from this theorem.
\end{proof}
\section{Quantitative Estimates of the Sampling Error}

\label{sec:quantitative} In the previous  Theorem \ref{thm:consistency} we have proven that
the sampling error converges to $0$ almost surely and that this implies the a.s.\ convergence of the generates distribution in the Jensen-Shannon distance. This worked for rather weak assumptions on the regularity. In particular, the regularity assumptions $k\geq 1$, $0<\alpha\leq 1$ were independent of the dimension $d$. However, we have not obtained any results on the rate of convergence.

The main objective
of this section is to estimate the rate of that convergence. To achieve this, we make strong assumptions on the differentiability of generators and discriminators. Especially, if the dimension $d$ is high -- in practical applications $d$ can be in the millions -- we have to assume that the density $f_\mu$ of $\mu$, the generators $\varphi$ and discriminators $\xi$ are `almost'
 $C^\infty$-functions.

To achieve an upper bound for the expectation of the sampling error\linebreak{}
$\varepsilon_{\textup{sample}}\left(n\right)$, i.e., for the quantity 
\[
\mathbb{E}\left[\sup_{\substack{\varphi\in\mathscr{H}_{K,\hat{K}}^{G}\\
\xi\in\mathscr{H}_{B,C_{1},C_{2}}^{D}
}
}\left|\hat{L}\left(\varphi,\xi,n\right)-L\left(\varphi,\xi\right)\right|\right],
\]
we apply the Dudley estimate \eqref{eq:entr_sum_int} to the
random processes $\pm\bigl[\hat{L}\left(\varphi,\xi,n\right)$\linebreak{}
$-L\left(\varphi,\xi\right)\bigr]$. The so-called empirical process over $\mathscr{H}^G_{K,\hat{K}}\times\mathscr{H}^D_{B,B_1,C_2}$, 
\[
\mathscr{H}^G_{K,\hat{K}}\times\mathscr{H}^D_{B,B_1,C_2}\ni(\varphi,\xi)\mapsto
\hat{L}\left(\varphi,\xi,n\right)-L\left(\varphi,\xi\right),
\]
is clearly centered. We already know from the previous section that $L\left(\cdot,\cdot\right)$
and $\hat{L}\left(\cdot,\cdot,n\right)$ are continuous on $\mathscr{H}_{K,\hat{K}}^{G}\times\mathscr{H}_{B,C_{1},C_{2}}^{D}$. To apply Dudley's inequality,
it remains to check the subgaussian property, see Definition \ref{def:Subgaussian}, of the increments of the empirical process
$\hat{L}(\varphi,\xi,n)-L(\varphi,\xi)$.
\begin{lem}
For each positive integer $n$, the increment $\hat{L}\left(\varphi_{1},\xi_{1},n\right)-$\linebreak{}
$L(\varphi_1,\xi_1)-\hat{L}\left(\varphi_{2},\xi_{2},n\right)+L(\varphi_2,\xi_2)$ is $\left[\rho_{n}\left(\left(\varphi_{1},\xi_{1}\right),\left(\varphi_{2},\xi_{2}\right)\right)\right]^{2}$-subgaussian
with
\[
\rho_{n}\left(\left(\varphi_{1},\xi_{1}\right),\left(\varphi_{2},\xi_{2}\right)\right)\coloneqq\frac{2}{B\sqrt{n}}\left(\left\Vert \xi_{1}-\xi_{2}\right\Vert _{\infty}+C_{1}\left\Vert \varphi_{1}-\varphi_{2}\right\Vert _{\infty}\right).
\]
\end{lem}

\begin{proof}
It is easy to check that the arithmetic mean of $n$ independent and
identically distributed $\sigma^{2}$-subgaussian random variables
is $\frac{\sigma^{2}}{n}$-subgaussian. This fact, together with the
representation
\begin{multline*}
\hat{L}\left(\varphi_{1},\xi_{1},n\right)-\hat{L}\left(\varphi_{2},\xi_{2},n\right)=\dfrac{1}{2n}\sum_{i=1}^{n}\left[\log\xi_{1}\left(Y_{i}\right)-\log\xi_{2}\left(Y_{i}\right)\right]\\
+\dfrac{1}{2n}\sum_{i=1}^{n}\left\{ \log\left[1-\xi_{1}\left(\varphi_{1}\left(Z_{i}\right)\right)\right]-\log\left[1-\xi_{2}\left(\varphi_{2}\left(Z_{i}\right)\right)\right]\right\} ,
\end{multline*}
shows that it suffices to prove the subgaussian property in the case
$n=1$.

Next, we will estimate the deviation $\bigl\Vert\hat{L}\left(\varphi_{1},\xi_{1},1\right)-\hat{L}\left(\varphi_{2},\xi_{2},1\right)\bigr\Vert_{\infty}$.
We have
\begin{align*}
\left|\hat{L}\left(\varphi_{1},\xi_{1},1\right)-\hat{L}\left(\varphi_{2},\xi_{2},1\right)\right| & \leq\dfrac{1}{2}\bigl|\log\xi_{1}\left(Y_{1}\right)-\log\xi_{2}\left(Y_{1}\right)\bigr|\\
 & +\dfrac{1}{2}\left|\log\left(1-\xi_{1}\left(Y_{1}^{(\varphi_{1})}\right)\right)-\log\left(1-\xi_{2}\left(Y_{1}^{\left(\varphi_{1}\right)}\right)\right)\right|\\
 & +\dfrac{1}{2}\left|\log\left(1-\xi_{2}\left(Y_{1}^{\left(\varphi_{1}\right)}\right)\right)-\log\left(1-\xi_{2}\left(Y_{1}^{\left(\varphi_{2}\right)}\right)\right)\right|.
\end{align*}
The inequalities $B\leq\xi_{i}\leq1-B\;\:(i=1,2)$, together with
the elementary estimate
\[
\left|\log x_{1}-\log x_{2}\right|=\left|\int\limits _{x_{1}}^{x_{2}}\frac{dt}{t}\right|\leq\frac{\left|x_{1}-x_{2}\right|}{\min\left\{ x_{1},x_{2}\right\} }\quad\left(x_{1},x_{2}>0\right),
\]
yield
\begin{multline*}
\left|\hat{L}\left(\varphi_{1},\xi_{1},1\right)-\hat{L}\left(\varphi_{2},\xi_{2},1\right)\right|\leq\frac{1}{2B}\left[\left|\xi_{1}\left(Y_{1}\right)-\xi_{2}\left(Y_{1}\right)\right|\right.\\
\left.+\left|\xi_{1}\left(Y_{1}^{\left(\varphi_{1}\right)}\right)-\xi_{2}\left(Y_{1}^{\left(\varphi_{1}\right)}\right)\right|+\left|\xi_{2}\left(Y_{1}^{\left(\varphi_{1}\right)}\right)-\xi_{2}\left(Y_{1}^{\left(\varphi_{2}\right)}\right)\right|\right].
\end{multline*}
The first and the second summands on the right hand side do not exceed
$\left\Vert \xi_{1}-\xi_{2}\right\Vert _{\infty}$. To find an upper
bound for the third summand, we first note that
\[
\left|\xi_{2}(x_{1})-\xi_{2}(x_{2})\right|\leq\left\Vert D\xi_{2}\right\Vert _{\infty}\cdot\left|x_{1}-x_{2}\right|\leq C_{1}\left|x_{1}-x_{2}\right|\quad\left(x_{1},x_{2}\in[0,1]^{d}\right),
\]
hence $\left\Vert \xi_{2}\circ\varphi_{1}-\xi_{2}\circ\varphi_{2}\right\Vert _{\infty}\leq C_{1}\left\Vert \varphi_{1}-\varphi_{2}\right\Vert _{\infty}$.
Thus, we arrive at the estimate
\[
\left|\hat{L}\left(\varphi_{1},\xi_{1},1\right)-\hat{L}\left(\varphi_{2},\xi_{2},1\right)\right| \leq\dfrac{1}{B}\left[\left\Vert \xi_{1}-\xi_{2}\right\Vert _{\infty}+C_{1}\left\Vert \varphi_{1}-\varphi_{2}\right\Vert _{\infty}\right]=\frac{1}{2}\rho_{1}\left(\left(\varphi_{1},\xi_{1}\right),\left(\varphi_{2},\xi_{2}\right)\right).
\]
From this estimate, the estimate $\left|L(\varphi_1,\xi_1)-L(\varphi_2,\xi_2) \right|\leq \frac{1}{2}\rho_{1}\left(\left(\varphi_{1},\xi_{1}\right),\left(\varphi_{2},\xi_{2}\right)\right)$ immediately follows by taking the expected value on the left hand side. This, together with Hoeffding's Lemma (see \cite[Sec. 3.1]{Handel16}),
shows that
$\hat{L}\left(\varphi_{1},\xi_{1},1\right)-\hat{L}\left(\varphi_{2},\xi_{2},1\right)-L(\varphi_1,\xi_1)+L(\varphi_2,\xi_2)$
is $\left[\rho_{1}\left(\left(\varphi_{1},\xi_{1}\right),\left(\varphi_{2},\xi_{2}\right)\right)\right]^{2}$-subgaussian.
\end{proof}
An essential component in the proof of our main result (see Theorem
\ref{thm:sampl_err_estim}) is the covering of the closed unit ball
in the space $C^{k,\alpha}\left(\overline{U}\right)$ by balls of
a small radius $\varepsilon$ in the space $C\left(\overline{U}\right)$.

For a metric space $\left(T,\rho\right)$ and for $\varepsilon>0$,
$N\left(T,\rho,\varepsilon\right)$ will denote the \emph{covering
number} of $T$, i.e., the minimum cardinality of an $\varepsilon$-net
for $T$. The closed ball of radius $\varepsilon$ and center $x\in T$
will be denoted by $B_{T}\left[x,\varepsilon\right]$ (or simply $B\left[x,\varepsilon\right]$).
Below we will consider covering numbers only for Cauchy-precompact
metric spaces; this, in view of Hausdorff's theorem (see, e.g., \cite[Lemma I.6.15]{Dunf88vol1}),
guarantees the finiteness of all covering numbers.
\begin{thm}
If $U\subset\mathbb{R}^{d_1}$ is bounded, open and convex, then
\begin{equation}
\log N\left(B_{C^{k,\alpha}\left(\overline{U},\mathbb{R}^{d_{2}}\right)}[0,1],\left\Vert \cdot\right\Vert _{\infty},\varepsilon\right)\leq\gamma_{1}\varepsilon^{-\frac{d_{1}}{\alpha+k}},\label{eq:Hoel_cnum_cnorm}
\end{equation}
where $\gamma_{1}$ is a constant depending only on $d_{1},d_{2}$
and $\alpha+k$.
\end{thm}

\begin{proof}
In the case $d_{2}=1$ the estimate \eqref{eq:Hoel_cnum_cnorm} is
proved in \cite[Th. 2.7.1]{Vaart96}. The general case follows without
difficulty; to construct an $\varepsilon$-net with desired properties,
we first embed the unit ball $B_{C^{k,\alpha}\left(\overline{U},\mathbb{R}^{d_{2}}\right)}[0,1]$
into $\left[B_{C^{k,\alpha}\left(\overline{U},\mathbb{R}\right)}[0,1]\right]^{d_{2}}$
and then cover the ball $B_{C^{k,\alpha}\left(\overline{U},\mathbb{R}\right)}[0,1]$
by the balls $B_{C^{k,\alpha}\left(\overline{U},\mathbb{R}\right)}^{\infty}\left[f_{j},\tfrac{\varepsilon}{\sqrt{d_{2}}}\right]\;\,\left(j\in J\right)$
in the space $\left(C^{k,\alpha}\left(\overline{U},\mathbb{R}\right),\left\Vert \cdot\right\Vert _{\infty}\right)$
such that the cardinality $\left|J\right|$ of the index set $J$
satisfies the relation 
\[
\log\left|J\right|=\log N\left(B_{C^{k,\alpha}\left(\overline{U},\mathbb{R}\right)}[0,1],\left\Vert \cdot\right\Vert _{\infty},\tfrac{\varepsilon}{\sqrt{d_{2}}}\right)\leq\gamma_{1}^{*}\left(\tfrac{\varepsilon}{\sqrt{d_{2}}}\right)^{-\frac{d_{1}}{\alpha+k}}
\]
with a constant $\gamma_{1}^{*}$, depending only on $d_{1}$ and
$\alpha+k$. For each $l=\left(j_{1},\ldots,j_{d_{2}}\right)\in J^{d_{2}}$,
let $B_{C^{k,\alpha}\left(\overline{U},\mathbb{R}^{d_{2}}\right)}^{\infty}\left[f^{l},\varepsilon\right]$
denote the ball of radius $\varepsilon$ and centered at $f^{l}\coloneqq(f_{j_{1}},\ldots,f_{j_{d_{2}}})$
in the space $\left(C^{k,\alpha}\left(\overline{U},\mathbb{R}^{d_{2}}\right),\left\Vert \cdot\right\Vert _{\infty}\right)$.
Then 
\[
\left[B_{C^{k,\alpha}\left(\overline{U},\mathbb{R}\right)}[0,1]\right]^{d_{2}} \subset\bigcup_{j_{1},\ldots,j_{d_{2}}\in J}\prod_{m=1}^{d_{2}}B_{C^{k,\alpha}\left(\overline{U},\mathbb{R}\right)}^{\infty}\left[f_{j_{m}},\tfrac{\varepsilon}{\sqrt{d_{2}}}\right]\subset\bigcup_{l\in J^{d_{2}}}B_{C^{k,\alpha}\left(\overline{U},\mathbb{R}^{d_{2}}\right)}^{\infty}\left[f^{l},\varepsilon\right],
\]
\[
\log\bigl|J^{d_{2}}\bigr|=d_{2}\log\left|J\right|\leq\gamma_{1}^{*}d_{2}^{1+\frac{d_{1}}{2\left(\alpha+k\right)}}\varepsilon^{-\frac{d_{1}}{\alpha+k}},
\]
and we may take $\gamma_{1}\coloneqq\gamma_{1}^{*}d_{2}^{1+\frac{d_{1}}{2\left(\alpha+k\right)}}$.
\end{proof}

Dudley's metric entropy estimate \eqref{eq:entr_sum_int} requires the square root of the right hand side in \eqref{eq:Hoel_cnum_cnorm} to be integrable at zero. From this, we derive our regularity requirements and obtain a rate estimate for the expected value of the sampling error.

\begin{prop}
\label{prop:expec_error}If $k>1-\alpha+\frac{d}{2}$, then there
exists a positive constant $\gamma$, depending only on $d,\alpha,$
and $k$, such that
\begin{equation}
\mathbb{E}\left[\varepsilon_{\textup{sample}}\left(n\right)\right]\leq\frac{\gamma}{B}\max\left\{ C_{1}K,C_{2}\right\} n^{-\frac{1}{2}}\quad\left(n=1,2,\ldots\right),\label{eq:expec_error2}
\end{equation}
for each $K>0,\hat{K}>0,C_{i}>0\;\,\,(i=1,2)$ and $B\in\left(0,\frac{1}{2}\right)$.
\end{prop}

\begin{proof}
Since $\pm\left[\hat{L}\left(\varphi,\xi,n\right)-L\left(\varphi,\xi\right)\right]$
is a  continuous subgaussian processes on the compact space $\left(\Theta;\rho_{n}\right)=\left(\mathscr{H}_{K,\hat{K}}^{G}\times\mathscr{H}_{B,C_{1},C_{2}}^{D};\rho_{n}\right)$,
the entropy bound \eqref{eq:entr_sum_int} is applicable to $\hat{L}\left(\varphi,\xi,n\right)-L\left(\varphi,\xi\right)$,
which gives
\begin{align}
\mathbb{E}\left[\varepsilon_{\textup{sample}}\left(n\right)\right] & \leq12\int\limits _{0}^{\infty}\sqrt{\log N\left(\Theta,\rho_{n},\varepsilon\right)}d\varepsilon=12\int\limits _{0}^{\infty}\sqrt{\log N\left(\Theta,\rho_{1},\sqrt{n}\varepsilon\right)}d\varepsilon\nonumber \\
 & =\frac{12}{\sqrt{n}}\int\limits _{0}^{\infty}\sqrt{\log N\left(\Theta,\rho_{1},\varepsilon\right)}d\varepsilon.\label{eq:chain_bound}
\end{align}
It is easy to see that
\begin{align*}
N\left(\Theta,\rho_{1},\rho_{r_{1},r_{2}}\right) & \leq N\left(\mathscr{H}_{K,\hat{K}}^{G},\left\Vert \cdot\right\Vert _{\infty},r_{1}\right)N\left(\mathscr{H}_{B,C_{1},C_{2}}^{D},\left\Vert \cdot\right\Vert _{\infty},r_{2}\right)\\
 & \leq N\left(B_{C^{k,\alpha}}[0,K],\left\Vert \cdot\right\Vert _{\infty},r_{1}\right)N\left(B_{C^{k-1,\alpha}}[0,C_{2}],\left\Vert \cdot\right\Vert _{\infty},r_{2}\right),
\end{align*}
where $B_{C^{l,\alpha}}[0,r]$ denotes the closed ball of radius $r>0$
centered at $0$ in the space $C^{l,\alpha}\left([0,1]^{d}\right)\;\,\left(l=k-1,k\right)$
and $\rho_{r_{1},r_{2}}\coloneqq 4B^{-1}\left(r_{2}+C_{1}r_{1}\right)$.
The equality
\[
N\left(B_{C^{l,\alpha}\left([0,1]^{d}\right)}[0,r],\left\Vert \cdot\right\Vert _{\infty},\varepsilon\right)=N\left(B_{C^{l,\alpha}\left([0,1]^{d}\right)}[0,1],\left\Vert \cdot\right\Vert _{\infty},\varepsilon r^{-1}\right)\quad\left(l=k-1,k\right),
\]
together with \eqref{eq:Hoel_cnum_cnorm}, yields
\[
\log N\left(B_{C^{l,\alpha}\left([0,1]^{d}\right)}[0,r],\left\Vert \cdot\right\Vert _{\infty},\varepsilon\right)\leq\gamma_{1}\left(d,\alpha+l\right)\left(r\varepsilon^{-1}\right)^{\frac{d}{\alpha+l}}\quad\left(r>0\right).
\]
Using this inequality, we may estimate further
\[
\log N\left(\Theta,\rho_{1},\rho_{r_{1},r_{2}}\right)\leq\gamma_{2}\left(d,\alpha,k\right)\left[\left(Kr_{1}^{-1}\right)^{\frac{d}{\alpha+k}}+\left(C_{2}r_{2}^{-1}\right)^{\frac{d}{\alpha+k-1}}\right]
\]
with $\gamma_{2}\left(d,\alpha,k\right)\coloneqq\max\left\{ \gamma_{1}\left(d,\alpha+k\right),\gamma_{1}\left(d,\alpha+k-1\right)\right\} $.
Taking
\[
r_{1}=\frac{\varepsilon B}{4C_{1}},\quad r_{2}=\frac{\varepsilon B}{4}
\]
and putting
\[
\gamma_{3}\coloneqq\left[\gamma_{2}\left(d,\alpha,k\right)\max\Big\{\!\left(4C_{1}KB^{-1}\right)^{\frac{d}{\alpha+k}},\left(4C_{2}B^{-1}\right)^{\frac{d}{\alpha+k-1}}\Big\}\right]^{\frac{1}{2}},
\]
we obtain 
\begin{equation}
\sqrt{\log N\left(\Theta,\rho_{1},\varepsilon\right)}\leq\gamma_{3}\Big[\varepsilon^{-\frac{d}{2\left(\alpha+k\right)}}+\varepsilon^{-\frac{d}{2\left(\alpha+k-1\right)}}\Big].\label{eq:cov_num}
\end{equation}
We would like to integrate both parts of the last inequality. In view
of the condition $k>1-\alpha+\frac{d}{2}$, we have $\frac{d}{2\left(\alpha+k-1\right)}<1$
which guaranties the integrability of the right hand side near $0$.
Put $\delta_{1}\coloneqq\operatorname{diam}\left(\mathscr{H}_{K,\hat{K}}^{G}\times\mathscr{H}_{B,C_{1},C_{2}}^{D};\rho_{1}\right)$.
Since $N\left(\Theta,\rho_{1},\varepsilon\right)=1\;\,\left(\varepsilon>\delta_{1}\right),$
hence \eqref{eq:chain_bound} and \eqref{eq:cov_num} give
\[
\mathbb{E}\left[\varepsilon_{\textup{sample}}\left(n\right)\right] \leq\frac{12\gamma_{3}}{\sqrt{n}}\int\limits _{0}^{\delta_{1}}\Big[\varepsilon^{-\frac{d}{2\left(\alpha+k\right)}}+\varepsilon^{-\frac{d}{2\left(\alpha+k-1\right)}}\Big]d\varepsilon =\frac{12\gamma_{3}}{\sqrt{n}}\Big[\delta_{1}^{1-\frac{d}{2\left(\alpha+k\right)}}+\delta_{1}^{1-\frac{d}{2\left(\alpha+k-1\right)}}\Big];
\]
therefore we may take $\gamma_{4}=12\gamma_{3}\Big[\delta_{1}^{1-\frac{d}{2\left(\alpha+k\right)}}+\delta_{1}^{1-\frac{d}{2\left(\alpha+k-1\right)}}\Big]$
to obtain
\[
\mathbb{E}\left[\varepsilon_{\textup{sample}}\left(n\right)\right]\leq\gamma_{4}n^{-\frac{1}{2}}\quad\left(n=1,2,\ldots\right).
\]
Since $\alpha+k-1>\frac{d}{2}$, we can estimate
\[
\gamma_{4}\leq 48B^{-1}\gamma_{2}^{\frac{1}{2}}\max\left\{ C_{1}K,C_{2}\right\} \Big[\delta_{1}^{1-\frac{d}{2\left(\alpha+k\right)}}+\delta_{1}^{1-\frac{d}{2\left(\alpha+k-1\right)}}\Big],
\]
therefore we may choose $\gamma\coloneqq 48\gamma_{2}^{\frac{1}{2}}\Big[\delta_{1}^{1-\frac{d}{2\left(\alpha+k\right)}}+\delta_{1}^{1-\frac{d}{2\left(\alpha+k-1\right)}}\Big]$.
\end{proof}

We now exploit McDiarmid's concentration inequality (Theorem \ref{thm:MCDiarmid}) to obtain the following result in the spirit of PAC (probably approximately correct)-learning \cite{shalev2014understanding}:

\begin{thm}
\label{thm:samp_err_prob_estim} Let $0<\alpha\leq1$, and let $d$
and $k$ be positive integers such that $k>1-\alpha+\frac{d}{2}$.
There exists a positive constant $\gamma$, depending only on $d,\alpha,$
and $k$, such that for every $B\in\left(0,\frac{1}{2}\right),C_{i}>0\;(i=1,2)$,
for each $K,\hat{K}>0$ satisfying the condition $\phi\in\mathscr{H}_{K,\hat{K}}^{G}$,
for each positive integer $n$ and for every $t\geq0$ the following
estimate is true:
\begin{equation}
\mathbb{P}\left(\varepsilon_{\textup{sample}}\left(n\right)\geq t+\gamma B^{-1}\max\left\{ C_{1}K,C_{2}\right\} n^{-\frac{1}{2}}\right)\leq\exp\biggl(-\frac{nt^{2}}{\log^{2}B}\biggr).\label{eq:est_samp_error}
\end{equation}
\end{thm}

\begin{proof}
To obtain the desired upper bound, we will apply McDiarmid's inequality
\eqref{eq:McDiar_ineq}. In view of \eqref{eq:likelihoodDiscriminator}, the corresponding
function $f$ has the form
\[
f\left(y_{1},\ldots,y_{2n}\right)=\sup_{\substack{\varphi\in\mathscr{H}_{K,\hat{K}}^{G}\\
\xi\in\mathscr{H}_{B,C_{1},C_{2}}^{D}
}
}\left|f_{\varphi,\xi}\left(y_{1},\ldots y_{2n}\right)\right|
\]
with
\[
f_{\varphi,\xi}\left(y_{1},\ldots,y_{2n}\right)=\dfrac{1}{2n}\sum_{i=1}^{n}\log\xi\left(y_{i}\right)+\dfrac{1}{2n}\sum_{i=1}^{n}\log\left[1-\xi\left(\varphi\left(y_{n+i}\right)\right)\right]-L\left(\varphi,\xi\right).
\]
We need to estimate the oscillations of $f\left(y_{1},\ldots y_{2n}\right)$
with respect to each of its arguments. Let us do that for the first argument.
Using the triangle inequality and the estimate $B\leq\xi\leq1-B$,
we have
\begin{align*}
f\left(y''_{1},y_{2},\ldots,y_{2n}\right) & =\sup_{\substack{\varphi\in\mathscr{H}_{K,\hat{K}}^{G}\\
\xi\in\mathscr{H}_{B,C_{1},C_{2}}^{D}
}
}\negthickspace\left|f_{\varphi,\xi}\left(y'_{1},y_{2},\ldots,y_{2n}\right)+\dfrac{1}{2n}\log\xi\left(y''_{1}\right)-\dfrac{1}{2n}\log\xi\left(y'_{1}\right)\right|\\
 & \leq\sup_{\substack{\varphi\in\mathscr{H}_{K,\hat{K}}^{G}\\
\xi\in\mathscr{H}_{B,C_{1},C_{2}}^{D}
}
}\negthickspace\negthickspace\left|f_{\varphi,\xi}\left(y'_{1},y_{2},\ldots,y_{2n}\right)\right|+\dfrac{1}{2n}\sup_{\xi\in\mathscr{H}_{B,C_{1},C_{2}}^{D}}\negthickspace\negthickspace\left[\left|\log\xi\left(y''_{1}\right)\right|+\left|\log\xi\left(y'_{1}\right)\right|\right]\\
& \leq f\left(y'_{1},y_{2},\ldots,y_{2n}\right)+\dfrac{1}{2n}\cdot\left(-2\log B\right),
\end{align*}
hence
\[
f\left(y''_{1},y_{2},\ldots,y_{2n}\right)-f\left(y'_{1},y_{2},\ldots,y_{2n}\right)\leq-\frac{\log B}{n}.
\]
Thus, the oscillation of $f\left(y_{1},\ldots y_{2n}\right)$ with
respect to its first argument does not exceed $-\frac{\log B}{n}$;
and the same is true for all other arguments of $f$. Hence \eqref{eq:McDiar_ineq}
gives 
\[
\mathbb{P}\left(\varepsilon_{\textup{sample}}\left(n\right)-\mathbb{E}\left[\varepsilon_{\textup{sample}}\left(n\right)\right]\geq t\right)\leq e^{-\frac{nt^{2}}{\log^{2}B}}\quad\left(t\geq0\right).
\]
The latter, together with \eqref{eq:expec_error2}, implies the desired
estimate.
\end{proof}
In the next theorem, we let $K$ and $\hat{K}$ increase with
the sample size $n$ and we thereby adaptively increase the capacity of the hypothesis space of generators. This enables us to fulfill the requirement $\phi\in\mathscr{H}_{K,\hat{K}}^{G}$
for large $n$ without knowing $\kappa=\inf_{x\in[0,1]^d} f_\mu(x)$ and $\|f_\mu\|_{C^{k,\alpha}}$, the minimum value of the density and its norm, explicitly. 

For
each $K>0$, the constants $B,C_{1},C_{2}$ will be chosen as in Lemma
\ref{lem:opt_discriminator}. Thereby also  the space of discriminators is enlarged such that the assumptions of Theorem \ref{thm:Decomposition} remain valid. 

To indicate the dependence of the quantities
$\varepsilon_{\textup{sample}}\left(n\right)$ and $\hat{\varphi}_{n}$ on
$K$ and $\hat{K}$, we use the notation $\varepsilon_{\textup{sample}}^{K,\hat{K}}\left(n\right)$
and $\hat{\varphi}_{n,K,\hat{K}}$, respectively.
\begin{thm}
\label{thm:sampl_err_estim}Let $0<\alpha\leq1$, let $\beta>0$,
let $d$ and $k$ be positive integers satisfying the condition $k>1-\alpha+\frac{d}{2}$,
and let $\bigl\{ K_{n}\bigr\}_{n=1}^{\infty},\bigl\{\hat{K}_{n}\bigr\}_{n=1}^{\infty}$
be sequences of positive numbers, tending to $\infty$, such that
$K_{n}\hat{K}_{n}=(\log n)^{2^{-k-5}d^{-1}\beta}\;\:(n\geq3)$ with
$\beta\geq1$. Then for almost every $\omega$ there exists a positive
integer $N=N(\omega)$ such that
\begin{equation}
\varepsilon_{\textup{sample}}^{K_{n},\hat{K}_{n}}\leq\Gamma n^{-\frac{1}{2}}(\log n)^{(2^{-k-5}+2^{-1})\beta}\quad(n>N)\label{eq:estim_smapl_error}
\end{equation}
and
\begin{equation}
d_{JS}\left(f_{\mu},f_{\hat{\varphi}_{n,K_{n},\hat{K}_{n}}}\right)\leq2\Gamma n^{-\frac{1}{2}}(\log n)^{(2^{-k-5}+2^{-1})\beta}\quad(n>N),\label{eq:estim_JS_emp}
\end{equation}
where $\Gamma$ is a constant, depending only on $d,\alpha,$ and
$k$.
\end{thm}

\begin{proof}
We may assume, without loss of generality, that $K_{n}>1$ and $\hat{K}_{n}>1$
for every $n\in\mathbb{N}$. \eqref{eq:est_samp_error} gives
\begin{equation}
\mathbb{P}\left(\varepsilon_{\textup{sample}}^{K_{n},\hat{K}_{n}}\geq t+\gamma A_{n}n^{-\frac{1}{2}}\right)\leq\exp\left(-nt^{2}\log^{-2}B_{n}\right)\quad\left(t\geq0\right),\label{eq:sampl_err_n_est1}
\end{equation}
where $A_{n}=\bigl[1+(d!)^{2}K_{n}^{d}\hat{K}_{n}^{d}\bigr]\bigl[4(d!)^{2}d^{k+2}K_{n}^{d}\hat{K}_{n}^{d}\bigr]^{2^{k+4}}$,
$B_{n}=\text{\ensuremath{\bigl[}1+(d!\ensuremath{)^{2}K_{n}^{d}\hat{K}_{n}^{d}\,\bigr]}}^{-1}$
(see Lemma \ref{lem:opt_discriminator}). Clearly, we may assume that
$\gamma\geq1$. Taking $t=\gamma A_{n}n^{-\frac{1}{2}}$ in \eqref{eq:sampl_err_n_est1}
yields
\[
\mathbb{P}\left(\varepsilon_{\textup{sample}}^{K_{n},\hat{K}_{n}}\geq2\gamma A_{n}n^{-\frac{1}{2}}\right) \leq\exp\left(-\gamma^{2}A_{n}^{2}\log^{-2}B_{n}\right) \leq\exp\left(-A_{n}^{2}\log^{-2}B_{n}\right)\quad\left(n\in\mathbb{N}\right).
\]
In view of the inequality $\log(1+x)<x\;\,(x>0)$,
\[
A_{n}^{2}\log^{-2}B_{n}\geq\bigl[4(d!)^{2}d^{k+2}K_{n}^{d}\hat{K}_{n}^{d}\bigr]^{2^{k+5}}\geq2\log n,
\]
which implies the convergence of the series 
\[
\sum\limits _{n=1}^{\infty}\mathbb{P}\Bigl(\varepsilon_{\textup{sample}}^{K_{n},\hat{K}_{n}}\geq2\gamma A_{n}n^{-\frac{1}{2}}\Bigr).
\]
The Borel--Cantelli lemma shows that for $\mathbb{P}$-almost every random parameter $\omega$ there
exists a positive integer $N=N(\omega)$ such that
\[
\varepsilon_{\textup{sample}}^{K_{n},\hat{K}_{n}}<2\gamma A_{n}n^{-\frac{1}{2}}\quad(n>N(\omega)).
\]
The latter, together with the estimate
\[
A_{n} \leq2(d!)^{2}\bigl[4(d!)^{2}d^{k+2}\bigr]^{2^{k+4}}\bigl(K_{n}\hat{K}_{n}\bigr)^{d(1+2^{k+4})} \leq2(d!)^{2}\bigl[4(d!)^{2}d^{k+2}\bigr]^{2^{k+4}}(\log n)^{(2^{-k-5}+2^{-1})\beta},
\]
implies \eqref{eq:estim_smapl_error}.

In view of \eqref{eq:theor_loss}, the terms $L(\hat{\varphi}_{n,K_{n},\hat{K}_{n}},\xi_{\hat{\varphi}_{n,K_{n},\hat{K}_{n}}}),\,L(\varphi^{*},\xi_{\varphi^{*}})$
of the error decomposition given in Theorem \ref{thm:Decomposition}
can be replaced by the Jensen-Shannon divergences $d_{JS}\left(f_{\mu},f_{\hat{\varphi}_{n,K_{n},\hat{K}_{n}}}\right),d_{JS}\left(f_{\mu},f_{\varphi^{*}}\right)$.
Since Theorem~\ref{thm:convenient} states that $f_{\mu}=f_{\varphi^{*}}$
and we know that $\varphi^{*}$ is a member of $\mathscr{H}_{K_{n},\hat{K}_{n}}^{G}$
given by the inverse Rosenblatt transformation $\phi$ (see Proposition
\ref{prop: Ros_is_generator}), Theorem \ref{thm:Decomposition} gives
\[
d_{JS}\left(f_{\mu},f_{\hat{\varphi}_{n,K_{n},\hat{K}_{n}}}\right)\leq2\varepsilon_{\textup{sample}}^{K_{n},\hat{K}_{n}}\left(n\right),
\]
which, together with \eqref{eq:estim_smapl_error}, implies \eqref{eq:estim_JS_emp}.
\end{proof}
Note that, under the additional condition $k>1-\alpha+\frac{d}{2}$,
Theorem \ref{thm:consistency}
follows from Theorem \ref{thm:sampl_err_estim}.
\begin{rem}
Let $0<\alpha\leq1$, and let $d$ and $k$ be positive integers satisfying
the condition $k>\frac{d}{2}-\alpha$. Using approximation of densities
by ReLU networks in Besov spaces, A. Uppal, S. Singh, and B. P\'{o}czos
\cite{uppal2020nonparametric} assume the existence of a sequence
of generators $\left\{ \varphi_{n}\right\} $ with certain properties
such that almost surely
\begin{equation}
d_{JS}\left(f_{\mu},f_{\varphi_{n}}\right)\leq C\left(d,\alpha,k\right)n^{-\frac{1}{2}}\sqrt{\log n}.\label{eq:JS_estim}
\end{equation}
In practical applications we usually deal with large data sets which
correspond to a large dimension $d$. If the regularity $k$ is increased
by one, i.e., if $k$ satisfies the condition $k>1+\frac{d}{2}-\alpha$,
then, in view of Theorem \ref{thm:sampl_err_estim}, the estimate
\eqref{eq:estim_JS_emp} holds. Note that the estimates \eqref{eq:estim_JS_emp}
and \eqref{eq:JS_estim} are essentially of the same nature. The advantage
of Theorem \ref{thm:sampl_err_estim} is that it describes a way to
construct a sequence $\left\{ \varphi_{n}\right\} _{n=1}^{\infty}$
with desired properties so that here we do not rely on assumptions.
Furthermore, the rates for H\"{o}lder generators given here are much
faster than the rate $\sim n^{-d^{-1}}$ obtained in \cite{biau2020wasserstein}
for the case of Wasserstein GAN, where the authors however do not
require H\"{o}lder regularity of the density and in addition provide approximations with neural networks.

Recently, better rates of convergence up to $\sim n^{-1}$ were suggested in \cite{belomnestry2021}. In fact, from the viewpoint of parametric statistics in finite dimensional spaces, where $\mu_{\hat\theta_n}$ estimates $\hat\mu_{\theta_0}$, a convergence $d_{JS}(\mu,\hat \mu_\theta)\sim n^{-\frac{1}{2}}$ is not optimal, as from maximum likelihood theory $\hat \theta_n-\theta_0\sim n^{-\frac{1}{2}}$ (see e.g.\ \cite{Ferguson.2017}) and thus $d_{JS}(\mu_{\theta_0},\mu_{\hat \theta_n})\sim n^{-1}$ is possible in this setting, as $\theta'\mapsto d_{JS}(\mu_{\theta_0},\mu_{theta'})$ will approximately behave quadratic around the minimum $\theta_0$. We therefore suggest that the rates obtained in this work can well be improved.   
\end{rem}

\section{Conclusion and Outlook}

\label{sec:CO} In this work, we have proven that generative adversarial
learning can be successfully pursued in the large data limit for generators
in $C^{k,\alpha}$-H\"{o}lder spaces if fulfilling a uniform bound.
The crucial observation was the realizability of arbitrary distributions
with $C^{k,\alpha}$-H\"{o}lder density, a suitable bound on the norm
within the so-defined hypothesis space and a consistent formulation
of the hypothesis space of discriminators. The key technical ingredients
were a thorough investigation of the analytical properties of the
Rosenblatt transformation based on an inverse function theorem for
H\"{o}lder spaces, which seems not to be known in the literature.
At the same time, H\"{o}lder spaces provide us with very flexible
possibilities for precompact embeddings, both qualitatively in $C^{k,\alpha'}$-spaces
and quantitatively in $L^{\infty}$ where explicit estimates on covering
numbers are known. Both kinds of embeddings were exploited in the
proofs of learnability with varying regularity requirements, without
and with explicit convergence rates.

From the insight generated through this work, some new research lines 
seem to be promising. While this paper focuses on the theory
of infinite dimensional generative learning, the understanding of
generative adversarial networks can profit from such an analysis.
In fact, deep (and shallow) neural networks (D(S)NN) on the one hand
have the universal approximation property which in recent times has
also been studied quantitatively \cite{yarotsky2017error,petersen2018optimal},
including deep convolutional neural networks \cite{yarotsky2018universal,petersen2020equivalence},
too. Thus, the approximation of the Rosenblatt transformation by DNN
seems to be feasible. DNN with smooth sigmoid activation functions
could provide a `conformal' approximation within $C^{k,\alpha}$.
Note that the rates of \cite{yarotsky2017error} are yet to be established
for this class of networks. ReLU Networks would correspond to an exterior,
`non-conformal' approximation, where the estimate of the sampling
error would require a revision in order to obtain convergence rates
that are independent of the network's weight count. While both strategies
seem feasible, a certain level of technicalities is to be expected
to achieve their implementation. Obviously, the `non-conformal' strategy
would be very valuable for the understanding of the success of the
contemporary GAN technology.

Secondly, the `triangular' structure of the Rosenblatt transformation
implies that the learning problem could be split into a hierarchy
of consecutive learning problems starting from one input and output
dimension and proceeding to $j$-dimensional input in the $j$-th
learning step, whereas always only one additional dimension of output
has to be learned at a time. Note that this does not necessarily have
to happen in a lexicographical order of input channels, but one could
easily combine this with linear transformations that consecutively
train multiscale hierarchies like wavelet coefficients of images.
Also this triangular structure is more easy to invert numerically,
which might be of interest in the construction of a CycleGAN \cite{Zhu.2018} (in moderate
dimension).

\appendix

\section{Analysis of Hölder Functions } \label{app:A}

In this section we collect some analytical results that are not directly connected to generative adversarial learning. In particular, we derive certain technical bounds for products and
quotients of H\"{o}lder differentiable functions, prove a version
of the inverse function theorem for H\"{o}lder differentiable maps, which might be of independent interest. 

\subsection{General properties of H\"{o}lder functions}

It is easy to see that $C^{k,\alpha}\left(\overline{U},\mathbb{R}^{d_{2}}\right)$
is a Banach space (see \cite[Sec. 4.1]{Gil01}); its members will
be referred to as \emph{$k$-times $\alpha$-H\"{o}lder differentiable}
functions on the set $\overline{U}$ (for a fixed $\alpha$, we shall
sometimes say H\"{o}lder differentiable rather than $\alpha$-H\"{o}lder
differentiable). Note that 
\begin{equation}
C^{k,\alpha_{2}}\left(\overline{U},\mathbb{R}^{d_{2}}\right)\subset C^{k,\alpha_{1}}\left(\overline{U},\mathbb{R}^{d_{2}}\right)\quad\left(k\geq0,\,0<\alpha_{1}<\alpha_{2}\leq1\right);\label{eq: Holder_inclusion}
\end{equation}
moreover, this embedding is compact, i.e., each bounded subset of
$C^{k,\alpha_{2}}\left(\overline{U},\mathbb{R}^{d_{2}}\right)$ is
relatively compact in $C^{k,\alpha_{1}}\left(\overline{U},\mathbb{R}^{d_{2}}\right)$
(see, e.g., \cite[Lemma 6.33]{Gil01}).

The special case $\alpha=1$ is of particular interest; the space
$C^{k,1}\left(\overline{U},\mathbb{R}^{d_{2}}\right)$ will be called
the \emph{Lipschitz space} and its members will be referred to as
\emph{$k$-times Lipschitz differentiable} functions on $\overline{U}$.
The mean value theorem for vector-valued functions (see, e.g., \cite[Th. 9.19]{rud76})
shows that $C^{1}\left(\overline{U},\mathbb{R}^{d_{2}}\right)\subset C^{0,1}\left(\overline{U},\mathbb{R}^{d_{2}}\right)$,
hence 
\begin{equation}
C^{k}\left(\overline{U},\mathbb{R}^{d_{2}}\right)\subset C^{k-1,1}\left(\overline{U},\mathbb{R}^{d_{2}}\right)\quad\left(k\geq1\right).\label{eq:cont_deriv_Lip}
\end{equation}
Moreover, the embedding \eqref{eq:cont_deriv_Lip} is continuous;
in fact, 
\[
\left\Vert f\right\Vert _{C^{k-1,1}\left(\overline{U},\mathbb{R}^{d_{2}}\right)}\leq2\left\Vert f\right\Vert _{C^{k}\left(\overline{U},\mathbb{R}^{d_{2}}\right)}\quad\left(f\in C^{k}\left(\overline{U},\mathbb{R}^{d_{2}}\right)\right).
\]

\subsection{Bounds for products and quotients of differentiable functions}
In this subsection we first compute $C^k$-bounds for products and quotients of differentiable functions in order to prepare the ground for the corresponding results for H\"{o}lder functions. 
\begin{prop}
\label{prop:ck_norm_prod}Let $U\subset\mathbb{R}^{d_{1}}$ be a bounded
open set, let $k$ be a nonnegative integer, and let $f_{i}\in C^{k}\left(\overline{U},\mathbb{R}\right)\;\,(1\leq i\leq m)$.
Then
\begin{equation}
\left\Vert f_{1}f_{2}\cdots f_{m}\right\Vert _{C^{k}\left(\overline{U},\mathbb{R}\right)}\leq m^{k}\left\Vert f_{1}\right\Vert _{C^{k}\left(\overline{U},\mathbb{R}\right)}\left\Vert f_{2}\right\Vert _{C^{k}\left(\overline{U},\mathbb{R}\right)}\cdots\left\Vert f_{m}\right\Vert _{C^{k}\left(\overline{U},\mathbb{R}\right)}.\label{eq:ck_norm_prod}
\end{equation}
\end{prop}

\begin{proof}
For simplicity of notation, let us write $C^{k}$ instead of $C^{k}\left(\overline{U},\mathbb{R}\right)$.
We proceed by induction on $k$. In the case $k=0$ the statement
is obvious. Now assume that the estimate is true for the norm $\left\Vert \cdot\right\Vert _{C^{k-1}}$
with $k\geq1$. To estimate the quantity $\left\Vert f_{1}f_{2}\cdots f_{m}\right\Vert _{C^{k}}$,
note that
\begin{equation}
\left\Vert h\right\Vert _{C^{k}}=\max\left\{ \left\Vert h\right\Vert _{C^{k-1}},\bigl\Vert h'_{x_{1}}\bigr\Vert_{C^{k-1}},\ldots,\bigl\Vert h'_{x_{d_{1}}}\bigr\Vert_{C^{k-1}}\right\} \quad\left(h\in C^{k}\right).\label{eq:Ck_norm_induct}
\end{equation}
For every $j=1,2,\ldots,d_{1}$ we have
\begin{align*}
\bigl\Vert(f_{1}f_{2}\cdots f_{m})'_{x_{j}}\bigr\Vert_{C^{k-1}} & =\biggl\Vert\sum_{i=1}^{m}\biggl(\tfrac{\partial f_{i}}{\partial x_{j}}\prod_{\substack{1\leq l\leq m\\
l\ne i
}
}f_{l}\biggr)\biggr\Vert_{C^{k-1}}\leq\sum_{i=1}^{m}\biggl\Vert\tfrac{\partial f_{i}}{\partial x_{j}}\prod_{\substack{1\leq l\leq m\\
l\ne i
}
}f_{l}\biggr\Vert_{C^{k-1}}\\
 & \leq m^{k-1}\sum_{i=1}^{m}\Bigl\Vert\tfrac{\partial f_{i}}{\partial x_{j}}\Bigr\Vert_{C^{k-1}}\prod_{\substack{1\leq l\leq m\\
l\ne i
}
}\left\Vert f_{l}\right\Vert _{C^{k-1}} \leq m\cdot m^{k-1}\prod_{l=1}^{m}\left\Vert f_{l}\right\Vert _{C^{k}},
\end{align*}
which, together with \eqref{eq:Ck_norm_induct}, yields the desired
estimate for $\left\Vert f_{1}f_{2}\cdots f_{m}\right\Vert _{C^{k}}$.
\end{proof}
\begin{lem}
\label{lem:ck_norm_recip}Let $U\subset\mathbb{R}^{d_{1}}$ be a bounded
open set, let $k$ be a nonnegative integer, and let $v\in C^{k}\left(\overline{U},\mathbb{R}\right)$
with $\inf\limits _{x\in U}\left|v(x)\right|>0$. Then \begingroup\abovedisplayskip=0.5ex
\[
\left\Vert \tfrac{1}{v}\right\Vert _{C^{k}}\leq4^{2^{k}-k-1}\Vert v\Vert_{C^{k}}^{2^{k}-1}\Bigl[\inf\limits _{x\in U}\left|v(x)\right|\Bigr]^{-2^{k}}.
\]
\endgroup
\end{lem}

\begin{proof}
The case $k=0$ is trivial. Let $k\geq1$. Applying \eqref{eq:Ck_norm_induct}
to $h=\tfrac{1}{v}$ gives
\[
\left\Vert \tfrac{1}{v}\right\Vert _{C^{m}}=\max\left\{ \left\Vert \tfrac{1}{v}\right\Vert _{C^{m-1}},\bigl\Vert v'_{x_{1}}\cdot\tfrac{1}{v^{2}}\bigr\Vert_{C^{m-1}},\ldots,\bigl\Vert v'_{x_{d_{1}}}\cdot\tfrac{1}{v^{2}}\bigr\Vert_{C^{m-1}}\right\} \;\left(1\leq m\leq k\right).
\]
In view of Proposition \ref{prop:ck_norm_prod},
\[
\bigl\Vert v'_{x_{j}}\cdot\tfrac{1}{v^{2}}\bigr\Vert_{C^{m-1}}\leq2^{m-1}\bigl\Vert v'_{x_{j}}\bigr\Vert_{C^{m-1}}\left\Vert \tfrac{1}{v^{2}}\right\Vert _{C^{m-1}}\leq4^{m-1}\Vert v\Vert_{C^{m}}\left\Vert \tfrac{1}{v}\right\Vert _{C^{m-1}}^{2}\;\,\left(j\leq d_{1}\right),
\]
which, together with the previous equality, yields
\[
\left\Vert \tfrac{1}{v}\right\Vert _{C^{m}}\leq\left\Vert \tfrac{1}{v}\right\Vert _{C^{m-1}}\max\left\{ 1,4^{m-1}\Vert v\Vert_{C^{m}}\left\Vert \tfrac{1}{v}\right\Vert _{C^{m-1}}\right\} .
\]
Proposition \ref{prop:ck_norm_prod} shows that
\[
1=\left\Vert v\cdot\tfrac{1}{v}\right\Vert _{C^{m-1}}\leq2^{m-1}\Vert v\Vert_{C^{m-1}}\left\Vert \tfrac{1}{v}\right\Vert _{C^{m-1}}\leq4^{m-1}\Vert v\Vert_{C^{m}}\left\Vert \tfrac{1}{v}\right\Vert _{C^{m-1}},
\]
hence the estimate for $\left\Vert \tfrac{1}{v}\right\Vert _{C^{m}}$
reduces to
\[
\left\Vert \tfrac{1}{v}\right\Vert _{C^{m}}\leq4^{m-1}\Vert v\Vert_{C^{m}}\left\Vert \tfrac{1}{v}\right\Vert _{C^{m-1}}^{2}\quad\left(1\leq m\leq k\right).
\]
Using this inequality and applying induction on $m$, we can easily
prove that
\begin{equation}
\left\Vert \tfrac{1}{v}\right\Vert _{C^{m}}\leq\Bigl[\inf\limits _{x\in U}\left|v(x)\right|\Bigr]^{-2^{m}}\prod_{j=1}^{m}\bigl(4^{j-1}\Vert v\Vert_{C^{j}}\bigr)^{2^{m-j}}\quad\left(1\leq m\leq k\right).\label{eq:Ck_norm_ratio_3}
\end{equation}
\eqref{eq:Ck_norm_ratio_3} implies 
\[
\left\Vert \tfrac{1}{v}\right\Vert _{C^{k}}\leq4^{k\sum\limits _{i=0}^{k-1}2^{i}-\sum\limits _{i=0}^{k-1}(i+1)2^{i}}\Vert v\Vert_{C^{k}}^{\sum\limits _{i=0}^{k-1}2^{i}}\Bigl[\inf\limits _{x\in U}\left|v(x)\right|\Bigr]^{-2^{k}}
\]
which, together with the equalities $\sum_{i=0}^{k-1}2^{i}=2^{k}-1$
and $\sum_{i=0}^{k-1}(i+1)2^{i}=(k-1)2^{k}+1$, yield the desired
estimate.
\end{proof}
The previous two results now easily combine to an estimate of the $C^k$-norm for quotients of $C^k$ functions.

\begin{cor}
\label{cor:ck_norm_frac}Let $U\subset\mathbb{R}^{d_{1}}$ be a bounded
open set, let $k$ be a non negative integer, and let $f,g\in C^{k}\left(\overline{U},\mathbb{R}\right)$
with $\inf\limits _{x\in U}\left|g(x)\right|>0$. Then
\[
\left\Vert \tfrac{f}{g}\right\Vert _{C^{k}}\leq2^{2^{k+1}-k-2}\Bigl[\inf\limits _{x\in U}\left|g(x)\right|\Bigr]^{-2^{k}}\left\Vert f\right\Vert _{C^{k}}\Vert g\Vert_{C^{k}}^{2^{k}-1}.
\]
\end{cor}

\begin{proof}
The proof follows from Proposition \ref{prop:ck_norm_prod} and Lemma
\ref{lem:ck_norm_recip}.
\end{proof}
\subsection{H\"{o}lder norm bounds for products and quotients of H\"{o}lder differentiable
functions}

Let $U\subset\mathbb{R}^{d_{1}}$ be a bounded open set, let $0<\alpha\leq1$,
and let $k$ be a nonnegative integer. For $f\in C^{k,\alpha}\left(\overline{U},\mathbb{R}^{d_{2}}\right)$
put
\[
[f]_{k,\alpha}\coloneqq\left\Vert f\right\Vert _{C^{k,\alpha}}-\left\Vert f\right\Vert _{C^{k}}=\max_{\left|\beta\right|=k}\sup_{\substack{x,y\in U\\
x\ne y
}
}\frac{\left|D_{\beta}f(x)-D_{\beta}f(y)\right|}{\left|x-y\right|^{\alpha}}.
\]

\begin{lem}
\label{lem:hoel_brac_prod}Let $U\subset\mathbb{R}^{d_{1}}$ be a
bounded open set, let $k$ be a nonnegative integer, let $0<\alpha\leq1$,
and let $f_{i}\in C^{k,\alpha}\left(\overline{U},\mathbb{R}\right)\;\,(1\leq i\leq m)$.
Then $f_{1}f_{2}\cdots f_{m}\in C^{k,\alpha}\left(\overline{U},\mathbb{R}\right)$
and
\[
[f_{1}f_{2}\cdots f_{m}]_{k,\alpha}\leq m^{k+1}\max_{1\leq i\leq m}\left\Vert f_{i}\right\Vert _{C^{k}}^{m-1}\cdot\max_{\substack{1\leq i\leq m\\
0\leq j\leq k
}
}[f_{i}]_{j,\alpha}\,.
\]
\end{lem}

\begin{proof}
In the case $k=0$ we use \eqref{eq:prod_change} to obtain
\begin{align*}
[f_{1}f_{2}\cdots f_{m}]_{0,\alpha} & \leq m\max_{1\leq i\leq m}\left\Vert f_{i}\right\Vert _{C}^{m-1}\cdot\max_{1\leq i\leq m}\sup_{\substack{x,y\in U\\
x\ne y
}
}\frac{\left|f_{i}(x)-f_{i}(y)\right|}{\left|x-y\right|^{\alpha}}\\
 & =m\max_{1\leq i\leq m}\left\Vert f_{i}\right\Vert _{C}^{m-1}\cdot\max_{1\leq i\leq m}[f_{i}]_{0,\alpha}\,.
\end{align*}
Now suppose the statement is true for $C^{k-1,\alpha}\left(\overline{U},\mathbb{R}\right)$.
It is clear that
\[
[f_{1}\cdots f_{m}]_{k,\alpha}=\max_{1\leq j\leq d_{1}}\max_{\left|\beta\right|=k-1}\sup_{\substack{x,y\in U\\
x\ne y
}
}\frac{\left|\bigl(D_{\beta}\tfrac{\partial}{\partial x_{j}}\bigr)(f_{1}\cdots f_{m})(x)-\bigl(D_{\beta}\tfrac{\partial}{\partial x_{j}}\bigr)(f_{1}\cdots f_{m})(y)\right|}{\left|x-y\right|^{\alpha}}.
\]
Assume that the first maximum in the right hand side is attained at
$j=j_{0}$. Since
\begin{multline*}
\left|\bigl(D_{\beta}\tfrac{\partial}{\partial x_{j_{0}}}\bigr)(f_{1}\cdots f_{m})(x)-\bigl(D_{\beta}\tfrac{\partial}{\partial x_{j_{0}}}\bigr)(f_{1}\cdots f_{m})(y)\right|\\
\leq\sum_{i=1}^{m}\biggl|D_{\beta}\biggl(\tfrac{\partial f_{i}}{\partial x_{j_{0}}}\prod_{\substack{1\leq l\leq m\\
l\ne i
}
}f_{l}\biggr)(x)-D_{\beta}\biggl(\tfrac{\partial f_{i}}{\partial x_{j_{0}}}\prod_{\substack{1\leq l\leq m\\
l\ne i
}
}f_{l}\biggr)(y)\biggr|,
\end{multline*}
we obtain
\[
[f_{1}f_{2}\cdots f_{m}]_{k,\alpha}\leq\sum_{i=1}^{m}\biggl[\tfrac{\partial f_{i}}{\partial x_{j_{0}}}\prod_{\substack{1\leq l\leq m\\
l\ne i
}
}f_{l}\biggr]_{k-1,\alpha}.
\]
The induction hypothesis shows that
\begin{align*}
\biggl[\tfrac{\partial f_{i}}{\partial x_{j_{0}}}\prod_{\substack{1\leq l\leq m\\
l\ne i
}
}f_{l}\biggr]_{k-1,\alpha}\leq & \,m^{k}\max\bigl\{\bigl\Vert\tfrac{\partial f_{i}}{\partial x_{j_{0}}}\bigr\Vert_{C^{k-1}}^{m-1},\max_{\substack{0\leq l\leq m\\
l\ne i
}
}\left\Vert f_{l}\right\Vert _{C^{k-1}}^{m-1}\bigr\}\cdot\max_{0\leq j\leq k-1}\bigl\{\bigl[\tfrac{\partial f_{i}}{\partial x_{j_{0}}}\bigr]_{j,\alpha},\max_{\substack{0\leq l\leq m\\
l\ne i
}
}[f_{l}]_{j,\alpha}\bigr\}\\
\leq & \,m^{k}\max_{0\leq l\leq m}\left\Vert f_{l}\right\Vert _{C^{k}}^{m-1}\max_{\substack{0\leq l\leq m\\
0\leq j\leq k
}
}[f_{i}]_{j,\alpha}\,,
\end{align*}
for every $j=1,2,\ldots,m$; which immediately implies the desired
estimate.
\end{proof}
\begin{prop}
\label{prop:hoel_norm_prod}Let $U\subset\mathbb{R}^{d_{1}}$ be a
bounded open set, let $k$ be a nonnegative integer, let $0<\alpha\leq1$,
and let $f_{i}\in C^{k,\alpha}\left(\overline{U},\mathbb{R}\right)\;\,(1\leq i\leq m)$.
Then $f_{1}f_{2}\cdots f_{m}\in C^{k,\alpha}\left(\overline{U},\mathbb{R}\right)$
and
\[
\left\Vert f_{1}f_{2}\cdots f_{m}\right\Vert _{C^{k,\alpha}}\leq2m^{k+1}\max\bigl\{1,[\operatorname{diam}(U)]^{1-\alpha}\bigr\}\max_{1\leq i\leq m}\left\Vert f_{i}\right\Vert _{C^{k}}^{m-1}\max_{1\leq i\leq m}\left\Vert f_{i}\right\Vert _{C^{k,\alpha}}.
\]
\end{prop}

\begin{proof}
The mean value theorem for vector-valued functions shows that
\begin{equation}
[f]_{j,\alpha}\leq[\operatorname{diam}(U)]^{1-\alpha}\left\Vert f\right\Vert _{C^{j+1}}\quad\left(f\in C^{j}\left(\overline{U},\mathbb{R}\right),0\leq j\leq k-1\right),\label{eq:Hoeld_bracket_Ck_estim1}
\end{equation}
hence
\begin{equation}
[f]_{j,\alpha}\leq\max\bigl\{1,[\operatorname{diam}(U)]^{1-\alpha}\bigr\}\left\Vert f\right\Vert _{C^{k,\alpha}}\;\,\left(f\in C^{k,\alpha}\left(\overline{U},\mathbb{R}\right),0\leq j\leq k\right),\label{eq:Hoeld_bracket_Ck_estim2}
\end{equation}
which, together with Lemma \ref{lem:hoel_brac_prod}, implies
\[
[f_{1}f_{2}\cdots f_{m}]_{k,\alpha}\leq m^{k+1}\max\bigl\{1,[\operatorname{diam}(U)]^{1-\alpha}\bigr\}\max_{1\leq i\leq m}\left\Vert f_{i}\right\Vert _{C^{k}}^{m-1}\max_{1\leq i\leq m}\left\Vert f_{i}\right\Vert _{C^{k,\alpha}}.
\]
The latter, together with Proposition \ref{prop:ck_norm_prod} and
with the equality
\begin{equation}
\left\Vert f\right\Vert _{C^{k,\alpha}\left(\overline{U},\mathbb{R}^{d_{2}}\right)}\coloneqq\left\Vert f\right\Vert _{C^{k}\left(\overline{U},\mathbb{R}^{d_{2}}\right)}+[f]_{k,\alpha}\quad\left(f\in C^{k,\alpha}\left(\overline{U},\mathbb{R}^{d_{2}}\right)\right),\label{eq:Hoel_norm_decomp}
\end{equation}
implies the required estimate.
\end{proof}
Using Proposition \ref{prop:hoel_norm_prod} and the definition of
the determinant, we obtain the following:
\begin{cor}
\label{cor:hoeld_norm_determinant}Let $U\subset\mathbb{R}^{d_{1}}$
be a bounded open set, let $k$ be a nonnegative integer, let $0<\alpha\leq1$,
and let $A=\left(a_{ij}\right)_{d_{2}\times d_{2}}\in C^{k,\alpha}\left(\overline{U},M_{d_{2}d_{2}}\left(\mathbb{R}\right)\right)$.
Then $\det(A)\in C^{k,\alpha}\left(\overline{U},\mathbb{R}\right)$
and
\begin{align*}
\Vert\det(A)\Vert_{C^{k,\alpha}}\leq & 2d_{2}!d_{2}^{k+1}\max\bigl\{1,[\operatorname{diam}(U)]^{1-\alpha}\bigr\}\max_{1\leq i,j\leq d_{2}}\left\Vert a_{ij}\right\Vert _{C^{k}}^{d_{2}-1}\max_{1\leq i,j\leq d_{2}}\left\Vert a_{ij}\right\Vert _{C^{k,\alpha}}.
\end{align*}
\end{cor}

\begin{lem}
\label{lem:hoel_norm_reciprocal}Let $U\subset\mathbb{R}^{d_{1}}$
be a bounded open set, let $k$ be a nonnegative integer, let $0<\alpha\leq1$,
and let $v\in C^{k,\alpha}\left(\overline{U},\mathbb{R}\right)$ with
$\inf\limits _{x\in U}\left|v(x)\right|>0$. Then $\tfrac{1}{v}\in C^{k,\alpha}\left(\overline{U},\mathbb{R}\right)$
and
\begin{align*}
\left\Vert \tfrac{1}{v}\right\Vert _{C^{k,\alpha}}\leq \,2^{k+1}\max\bigl\{1,[\operatorname{diam}(U)]^{1-\alpha}\bigr\}\max\bigl\{\left\Vert v\right\Vert _{C^{k}},2^{k}\left\Vert \tfrac{1}{v}\right\Vert _{C^{k}}^{2}\bigr\}\max\bigl\{\left\Vert v\right\Vert _{C^{k,\alpha}};2^{k}\left\Vert \tfrac{1}{v}\right\Vert _{C^{k}}^{2}\bigr\}.
\end{align*}
\end{lem}

\begin{proof}
We first prove that
\begin{equation}
\label{eq:hoel_brac_reciprocal_1}
\bigl[\tfrac{1}{v}\bigr]_{k,\alpha}\leq 2^{k}\max\bigl\{1,[\operatorname{diam}(U)]^{1-\alpha}\bigr\}\max\bigl\{\left\Vert v\right\Vert _{C^{k}},\left\Vert \tfrac{1}{v^{2}}\right\Vert _{C^{k}}\bigr\}\max\bigl\{\left\Vert v\right\Vert _{C^{k,\alpha}};\left\Vert \tfrac{1}{v^{2}}\right\Vert _{C^{k}}\bigr\}.
\end{equation}
In the case $k=0$, \eqref{eq:hoel_brac_reciprocal_1} follows from
the inequality $\bigl[\tfrac{1}{v}\bigr]_{0,\alpha}\leq\bigl\Vert\tfrac{1}{v^{2}}\bigr\Vert_{C}\cdot[v]_{0,\alpha}$.
Assuming $k\geq1$, the reasoning used in the proof of Lemma \ref{lem:hoel_brac_prod}
can be applied:
\begin{align*}
\bigl[\tfrac{1}{v}\bigr]_{k,\alpha} & =\max_{1\leq j\leq d_{1}}\max_{\left|\beta\right|=k-1}\sup_{\substack{x,y\in U\\
x\ne y
}
}\frac{\left|\bigl(D_{\beta}\tfrac{\partial}{\partial x_{j}}\bigr)\bigl(\tfrac{1}{v}\bigr)(x)-\bigl(D_{\beta}\tfrac{\partial}{\partial x_{j}}\bigr)\bigl(\tfrac{1}{v}\bigr)(y)\right|}{\left|x-y\right|^{\alpha}}\\
 & =\max_{1\leq j\leq d_{1}}\max_{\left|\beta\right|=k-1}\sup_{\substack{x,y\in U\\
x\ne y
}
}\frac{\left|D_{\beta}\bigl(v'_{x_{j}}\cdot\tfrac{1}{v^{2}}\bigr)(x)-D_{\beta}\bigl(v'_{x_{j}}\cdot\tfrac{1}{v^{2}}\bigr)(y)\right|}{\left|x-y\right|^{\alpha}}\\
 & =\max_{1\leq j\leq d_{1}}\bigl[v'_{x_{j}}\cdot\tfrac{1}{v^{2}}\bigr]_{k-1,\alpha}.
\end{align*}
Applying Lemma \ref{lem:hoel_brac_prod} to $\bigl[v'_{x_{j}}\cdot\tfrac{1}{v^{2}}\bigr]_{k-1,\alpha}$
yields
\begin{align*}
\bigl[\tfrac{1}{v}\bigr]_{k,\alpha}\leq & \,2^{k}\max_{1\leq j\leq d_{1}}\left[\max\bigl\{\bigl\Vert v'_{x_{j}}\bigr\Vert_{C^{k-1}},\left\Vert \tfrac{1}{v^{2}}\right\Vert _{C^{k-1}}\bigr\}\right.\\
 & \left.\;\,\cdot\max\bigl\{[v'_{x_{j}}]_{0,\alpha};\ldots;[v'_{x_{j}}]_{k-1,\alpha};[\tfrac{1}{v^{2}}]_{0,\alpha};\ldots;[\tfrac{1}{v^{2}}]_{k-1,\alpha}\bigr\}\right]\\
\leq & \,2^{k}\max\bigl\{\left\Vert v\right\Vert _{C^{k}},\left\Vert \tfrac{1}{v^{2}}\right\Vert _{C^{k-1}}\bigr\}\\
 & \;\,\cdot\max\bigl\{[v]_{1,\alpha};\ldots;[v]_{k,\alpha};[\tfrac{1}{v^{2}}]_{0,\alpha};\ldots;[\tfrac{1}{v^{2}}]_{k-1,\alpha}\bigr\}.
\end{align*}
The latter, combined with \eqref{eq:Hoeld_bracket_Ck_estim1} and
\eqref{eq:Hoeld_bracket_Ck_estim2}, implies \eqref{eq:hoel_brac_reciprocal_1}.

In view of \eqref{eq:ck_norm_prod}, $\left\Vert \tfrac{1}{v^{2}}\right\Vert _{C^{k}}\leq2^{k}\left\Vert \tfrac{1}{v}\right\Vert _{C^{k}}^{2}$,
which, together with \eqref{eq:hoel_brac_reciprocal_1}, gives
\begin{align}
\bigl[\tfrac{1}{v}\bigr]_{k,\alpha}\leq\, & 2^{k}\max\bigl\{1,[\operatorname{diam}(U)]^{1-\alpha}\bigr\}\max\bigl\{\left\Vert v\right\Vert _{C^{k}},2^{k}\left\Vert \tfrac{1}{v}\right\Vert _{C^{k}}^{2}\bigr\}\nonumber \\
 & \,\cdot\max\bigl\{\left\Vert v\right\Vert _{C^{k,\alpha}};2^{k}\left\Vert \tfrac{1}{v}\right\Vert _{C^{k}}^{2}\bigr\}.\label{eq:hoel_brac_reciprocal_2}
\end{align}
Next, by Proposition \ref{prop:ck_norm_prod},
\[
\left\Vert \tfrac{1}{v}\right\Vert _{C^{k}}=\left\Vert v\cdot\tfrac{1}{v^{2}}\right\Vert _{C^{k}}\leq2^{k}\left\Vert v\right\Vert _{C^{k}}\cdot\left\Vert \tfrac{1}{v^{2}}\right\Vert _{C^{k}}\leq2^{k}\left\Vert v\right\Vert _{C^{k}}\cdot2^{k}\left\Vert \tfrac{1}{v}\right\Vert _{C^{k}}^{2},
\]
hence $\left\Vert \tfrac{1}{v}\right\Vert _{C^{k}}$ does not exceed
the right hand side of \eqref{eq:hoel_brac_reciprocal_2}. This fact,
together with \eqref{eq:Hoel_norm_decomp} and \eqref{eq:hoel_brac_reciprocal_2},
implies the desired estimate for $\left\Vert \tfrac{1}{v}\right\Vert _{C^{k,\alpha}}$.
\end{proof}
Proposition \ref{prop:hoel_norm_prod} and Lemma \ref{lem:hoel_norm_reciprocal}
immediately imply the following
\begin{cor}
\label{cor:Hoeld_cont_frac}Let $u_{n}\to u$ and $v_{n}\to v$ in
$C^{k,\alpha}\left(\overline{U},\mathbb{R}\right)$-norm, and let
$\inf\limits _{x\in U,n\in\mathbb{N}}\left|v_{n}(x)\right|>0$. Then
$\frac{u_{n}}{v_{n}}\to\frac{u}{v}$ in $C^{k,\alpha}\left(\overline{U},\mathbb{R}\right)$-norm.
\end{cor}

Lemmas \ref{lem:ck_norm_recip}, \ref{lem:hoel_norm_reciprocal} and
Proposition \ref{prop:hoel_norm_prod} yield the following estimate
for the $C^{k,\alpha}$-norm of the quotient of two $C^{k,\alpha}$-functions:
\begin{prop}
\label{prop:Hoel_norm_frac}Let $U\subset\mathbb{R}^{d_{1}}$ be a
bounded open set, let $k$ be a nonnegative integer, let $0<\alpha\leq1$,
and let functions $f,g\in C^{k,\alpha}\left(\overline{U},\mathbb{R}\right)$
satisfy the conditions $\left\Vert f\right\Vert _{C^{k,\alpha}}\leq L,\left\Vert g\right\Vert _{C^{k,\alpha}}\leq L$
and $\inf\limits _{x\in U}\left|g(x)\right|\geq l^{-1}$ with real
numbers $L\geq1$ and $l\geq1$. Then \begingroup\abovedisplayskip=0.5ex
\[
\left\Vert \tfrac{f}{g}\right\Vert _{C^{k,\alpha}}\leq(2Ll)^{2^{k+5}}\max\bigl\{1,[\operatorname{diam}(U)]^{2(1-\alpha)}\bigr\}.
\]
\endgroup
\end{prop}

\begin{proof}
In view of Lemma \ref{lem:ck_norm_recip},
\[
\left\Vert \tfrac{1}{g}\right\Vert _{C^{k}}\leq2^{2^{k+1}-2k-2}\Bigl[\inf\limits _{x\in U}\left|g(x)\right|\Bigr]^{-2^{k}}\Vert g\Vert_{C^{k,\alpha}}^{2^{k}-1}\leq2^{2^{k+1}-2k}(Ll)^{2^{k}},
\]
hence
\[
\max\left\{ \left\Vert f\right\Vert _{C^{k}},\left\Vert \tfrac{1}{g}\right\Vert _{C^{k}}\right\} \leq2^{2^{k+1}}(Ll)^{2^{k}}\quad\text{and}\quad2^{k}\left\Vert \tfrac{1}{g}\right\Vert _{C^{k}}^{2}\leq2^{2^{k+2}}(Ll)^{2^{k+1}}.
\]
Therefore Lemma \ref{lem:hoel_norm_reciprocal} gives
\[
\left\Vert \tfrac{1}{g}\right\Vert _{C^{k,\alpha}}\leq\max\bigl\{1,[\operatorname{diam}(U)]^{1-\alpha}\bigr\}\cdot2^{2^{k+3}+k+1}(Ll)^{2^{k+2}}.
\]
The latter, together with Proposition \ref{prop:hoel_norm_prod},
yields
\begin{align*}
\left\Vert \tfrac{f}{g}\right\Vert _{C^{k,\alpha}} & \leq2^{k+2}\max\bigl\{1,[\operatorname{diam}(U)]^{2(1-\alpha)}\bigr\}\cdot2^{2^{k+1}}(Ll)^{2^{k}}\cdot2^{2^{k+3}+k+1}(Ll)^{2^{k+2}}\\
 & \leq(2Ll)^{2^{k+5}}\max\bigl\{1,[\operatorname{diam}(U)]^{2(1-\alpha)}\bigr\}.\tag*{\qedsymbol}
\end{align*}
\def\qedsymbol{}
\end{proof}
\subsection{Composition of differentiable functions}
Let $M_{d_{2}d_{3}}\left(\mathbb{R}\right)$ stand for the set of
all $d_{2}\times d_{3}$ matrices with real entries. The collection
of all matrix-valued functions $\left(a_{ij}\right)_{d_{2}\times d_{3}}:\overline{U}\to M_{d_{2}d_{3}}\left(\mathbb{R}\right)$,
whose entries are $k$-times H\"{o}lder differentiable on $\overline{U}$,
will be denoted by $C^{k,\alpha}\left(\overline{U},M_{d_{2}d_{3}}\left(\mathbb{R}\right)\right)$.
In the special case $d_{2}=1$ it is easily seen that the product
of H\"{o}lder continuous real-valued functions on $\overline{U}$
is again H\"{o}lder continuous (see, e.g., \cite[Sec. 4.1]{Gil01}),
hence the product of $k$-times H\"{o}lder differentiable real-valued
functions on $\overline{U}$ is again $k$-times H\"{o}lder differentiable
there. Moreover, the multiplication is continuous in $C^{k,\alpha}\left(\overline{U},\mathbb{R}\right)$-norm
(see Proposition \ref{prop:hoel_norm_prod}). Obviously, these assertions
remain true for matrix-valued functions. Moreover, the quotient of $k$-times
H\"{o}lder differentiable real-valued functions is also $k$-times
H\"{o}lder differentiable provided that the absolute value of the
denominator has a positive lower bound (see Lemma \ref{lem:hoel_norm_reciprocal}).

More interesting is the fact that superpositions of $k$-times H\"{o}lder
differentiable functions on $\overline{U}$ are $k$-times H\"{o}lder
differentiable:
\begin{lem}
\label{lem:Hoelder_composition}Let $U_{1}\subset\mathbb{R}^{d_{1}}$
and $U_{2}\subset\mathbb{R}^{d_{2}}$ be bounded open sets, let $k$
be a positive integer, and let $0<\alpha\leq1$.
\begin{enumerate}
\item[(a)] If $f\in C^{k,\alpha}\left(\overline{U}_{1},\mathbb{R}^{d_{2}}\right),f:U_{1}\to U_{2}$
and $g\in C^{k,\alpha}\left(\overline{U}_{2},\mathbb{R}^{d_{3}}\right)$,
then $g\circ f\in C^{k,\alpha}\left(\overline{U}_{1},\mathbb{R}^{d_{3}}\right)$.
\item[(b)] Let $f_{n}\in C^{k,\alpha}\left(\overline{U}_{1},\mathbb{R}^{d_{2}}\right),f_{n}:U_{1}\to U_{2},g_{n}\in C^{k,\alpha}\left(\overline{U}_{2},\mathbb{R}^{d_{3}}\right)\;\,\left(n\in\mathbb{N}\right)$,
and let $f_{n}\to f,\linebreak g_{n}\to g$ in $C^{k,\alpha}$-norm with $f$
satisfying the condition $f\left(U_{1}\right)\subset U_{2}$. Then
$g_{n}\circ f_{n}\to g\circ f$ in $C^{k,\alpha}$-norm.
\end{enumerate}
\end{lem}

The proof of Lemma \ref{lem:Hoelder_composition} is quite simple,
hence we will only outline it. First note that if $f\in C^{0,\alpha_{1}}\left(\overline{U}_{1},\mathbb{R}^{d_{2}}\right)$
and $g\in C^{0,\alpha_{2}}\left(\overline{U}_{2},\mathbb{R}^{d_{3}}\right)$
for some $\alpha_{1},\alpha_{2}\in\left(0,1\right]$, then ${g\circ f}\in C^{0,\alpha_{1}\alpha_{2}}\left(\overline{U}_{1},\mathbb{R}^{d_{3}}\right)$.
Next, we proceed inductively, using the equality $D^{m}\left(g\circ f\right)=D^{m-1}\left[\left(Dg\circ f\right)Df\right]\;\:\left(1\leq m\leq k\right)$.

In the considerations below we shall assume that $d_{1}=d_{2}=d$
and, for simplicity, we will write $C^{k,\alpha}\left(\overline{U}\right)$
instead of $C^{k,\alpha}\left(\overline{U},\mathbb{R}^{d_{2}}\right)$.

\subsection{The inverse function theorem for H\"{o}lder differentiable maps}
We are now in the position to prove the main theorem of this appendix, namely the inverse function theorem for H\"{o}lder differentiable functions on $\bar{U}\subset\mathbb{R}^d$. Note that this is different from the the inverse function theorem for Fr\'{e}chet differentiable functionals \emph{on} the Banach space of H\"{o}lder functions, which is a consequence of the Banach space version of the standard inverse function theorem.   

\begin{thm}
\label{Th: hoeld_diffeomor} Let $U\subset\mathbb{R}^{d}$ be a bounded
open set, let $k$ be a positive integer, and let ${0<\alpha\leq1}$.
\begin{enumerate}
\item[(a)]  If $\varphi:\overline{U}\to\overline{U}$ is a bijective and $k$-times
$\alpha$-H\"{o}lder differentiable function on $\overline{U}$ such
that the Jacobian determinant $J_{\varphi}$ satisfies the condition
\begin{equation}
\inf_{x\in U}\left|J_{\varphi}(x)\right|>0,\label{eq:Jac_posit}
\end{equation}
then $\varphi^{-1}\in C^{k,\alpha}\left(\overline{U}\right)$ and
\begin{equation}
\left\Vert D\left(\varphi^{-1}\right)\right\Vert _{\infty}\leq\tfrac{d!\left\Vert D\varphi\right\Vert _{\infty}^{d-1}}{\inf\limits _{x\in U}\left|J_{\varphi}(x)\right|}.\label{eq:estim_deriv_inv_func}
\end{equation}
\item[(b)]  If $\varphi_{n}:\overline{U}\to\overline{U}\;\,\left(n=1,2,\ldots\right)$
are bijective and $k$-times $\alpha$-H\"{o}lder differentiable functions
on $\overline{U}$ such that
\[
\inf_{x\in U,n\in\mathbb{N}}\left|J_{\varphi_{n}}(x)\right|>0
\]
and $\varphi_{n}\to\varphi$ in $C^{k,\alpha}$-norm, then $\varphi:\overline{U}\to\overline{U}$
is bijective, $\varphi^{-1}\in C^{k,\alpha}\left(\overline{U}\right)$
and $\varphi_{n}^{-1}\to\varphi^{-1}$ in $C^{k,\alpha}$-norm.
\end{enumerate}
\end{thm}

\begin{proof}
(a) We first show that $\left(D\varphi\right)^{-1}$ is $\left(k-1\right)$-times
H\"{o}lder differentiable on $\overline{U}$. Identifying $D\varphi$
with the Jacobian matrix of $\varphi$, we may write $\left(D\varphi\right)^{-1}$
in the form $\left(D\varphi\right)^{-1}=\frac{1}{J_{\varphi}}\operatorname{adj}\left(D\varphi\right)$,
where $\operatorname{adj}\left(D\varphi\right)$ is the adjugate matrix
of the Jacobian matrix of $\varphi$. The entries of $\operatorname{adj}\left(D\varphi\right)$
are minors of $D\varphi$ with corresponding sign factors. Since the
product of $\left(k-1\right)$-times H\"{o}lder differentiable functions
on $\overline{U}$ is $\left(k-1\right)$-times H\"{o}lder differentiable,
hence so is the matrix-valued function $\operatorname{adj}\left(D\varphi\right)$
(see Proposition \ref{prop:hoel_norm_prod}). In view of \eqref{eq:Jac_posit}
and Lemma \ref{lem:hoel_norm_reciprocal}, we conclude that $\left(D\varphi\right)^{-1}$
is $\left(k-1\right)$-times H\"{o}lder differentiable on $\overline{U}$.

Since $\overline{U}$ is compact and $\varphi:\overline{U}\to\overline{U}$
is a continuous bijection, $\varphi^{-1}$ is continuous on $\overline{U}$
(see, e.g., \cite[Lemma I.5.8]{Dunf88vol1}). The inverse function
theorem for $C^{1}$-maps (see, e.g., \cite[Th. 9.24]{rud76} or
\cite[Th. 3.1]{tay11}) implies that $\varphi^{-1}\in C^{1}\left(U\right)$.
If we differentiate both sides of the relation $\varphi\circ\varphi^{-1}=I$
and apply the chain rule, we obtain 
\begin{equation}
D\varphi^{-1}=\left(D\varphi\right)^{-1}\circ\varphi^{-1}.\label{eq:inv_deriv}
\end{equation}
Therefore $D\varphi^{-1}$ is continuous on $\overline{U}$, i.e.,
$\varphi^{-1}\in C^{1}\left(\overline{U}\right).$ Since $\left(D\varphi\right)^{-1}$
is $\left(k-1\right)$-times H\"{o}lder differentiable on $\overline{U}$
and $\varphi^{-1}$ is continuously differentiable on $\overline{U}$,
another application of \eqref{eq:inv_deriv} gives $\varphi^{-1}\in C^{2}\left(\overline{U}\right)$,
etc. Thus, we arrive in a finite number of steps at the relation $\varphi^{-1}\in C^{k}\left(\overline{U}\right)$.

In view of \eqref{eq: Holder_inclusion} and \eqref{eq:cont_deriv_Lip},
we have $\varphi^{-1}\in C^{k-1,\alpha}\left(\overline{U}\right)$.
Thus, $\left(D\varphi\right)^{-1}$ and $\varphi^{-1}$ are both $\left(k-1\right)$-times
H\"{o}lder differentiable on $\overline{U}$, hence their composition
also possesses that property, according to Lemma \ref{lem:Hoelder_composition};
therefore \eqref{eq:inv_deriv} implies that $\varphi^{-1}\in C^{k,\alpha}\left(\overline{U}\right)$.

To prove \eqref{eq:estim_deriv_inv_func}, we need an estimate for
the operator norm of an invertible matrix. If $A=\left(a_{ij}\right)_{d\times d}\in M_{dd}\left(\mathbb{R}\right)$
is invertible, then $A^{-1}=\frac{1}{\det\left(A\right)}\operatorname{adj}\left(A\right)$.
The operator norm and the determinant of any matrix $B=\left(b_{ij}\right)_{d\times d}\in M_{dd}\left(\mathbb{R}\right)$
can be easily estimated as 
\begin{equation}
\left|B\right|\leq d\max_{1\leq i,j\leq d}\left|b_{ij}\right|\label{eq:mat_norm_bound}
\end{equation}
and 
\begin{equation}
\left|\det\left(B\right)\right|\leq d!\left[\max_{1\leq i,j\leq d}\left|b_{ij}\right|\right]^{d}\negthinspace,\label{eq:det_up_bound}
\end{equation}
respectively. Therefore the absolute value of each entry of the matrix
$\operatorname{adj}\left(A\right)$ does not exceed $\left(d-1\right)!\Bigl[\max\limits _{1\leq i,j\leq d}\left|a_{ij}\right|\Bigr]^{d-1}\negthinspace,$
which, together with \eqref{eq:mat_norm_bound}, yields
\[
\left|A^{-1}\right|\leq\tfrac{d!\left[\max\limits _{1\leq i,j\leq d}\left|a_{ij}\right|\right]^{d-1}}{\left|\det\left(A\right)\right|}.
\]
Applying this estimate to $A=\left(D\varphi\right)(x)$ and using
\eqref{eq:inv_deriv}, we obtain \eqref{eq:estim_deriv_inv_func}.

(b) The convergence $\left\Vert \varphi_{n}-\varphi\right\Vert _{C^{k,\alpha}}\to0$
implies that $\left\Vert \varphi_{n}-\varphi\right\Vert _{C^{1}}\to0$,
hence
\[
C\coloneqq\tfrac{d!\sup\limits _{n\in\mathbb{N}}\left\Vert D\varphi_{n}\right\Vert _{\infty}^{d-1}}{\inf\limits _{x\in U,n\in\mathbb{N}}\left|J_{\varphi_{n}}(x)\right|}<\infty.
\]
\eqref{eq:estim_deriv_inv_func}, combined with the mean value theorem
for vector-valued functions, gives
\begin{equation}
\left|\varphi_{n}^{-1}\left(y_{1}\right)-\varphi_{n}^{-1}\left(y_{2}\right)\right|\leq C\left|y_{1}-y_{2}\right|\quad\left(y_{1},y_{2}\in\overline{U};n\in\mathbb{N}\right),\label{eq:phi_n_Lip}
\end{equation}
therefore
\[
\left|x_{1}-x_{2}\right|\leq C\left|\varphi_{n}(x_{1})-\varphi_{n}(x_{2})\right|\quad\left(x_{1},x_{2}\in\overline{U};n\in\mathbb{N}\right).
\]
Letting $n\to\infty$, we obtain the inequality $\left|x_{1}-x_{2}\right|\leq C\left|\varphi(x_{1})-\varphi(x_{2})\right|$\linebreak{}
$\left(x_{1},x_{2}\in\overline{U}\right)$, which implies the injectivity
of $\varphi$.

To prove the surjectivity of $\varphi$, choose any $y\in\overline{U}$.
Since $\varphi_{n}\left(\overline{U}\right)=\overline{U}$, there
exists $x_{n}\in\overline{U}$ such that $\varphi_{n}(x_{n})=y$,
for any $n\in\mathbb{N}$. The compactness of $\overline{U}$ guarantees
the existence of a convergent subsequence $\left\{ x_{n_{i}}\right\} _{i=1}^{\infty}.$
Put $x\coloneqq\lim\limits _{i\to\infty}x_{n_{i}}$. The convergence
in $C^{k,\alpha}$-norm implies the uniform convergence on $\overline{U}$,
hence $\varphi_{n_{i}}\left(x_{n_{i}}\right)-\varphi\left(x_{n_{i}}\right)\to0$.
Thus, 
\[
\varphi(x)=\varphi\left(\lim\limits _{i\to\infty}x_{n_{i}}\right)=\lim\limits _{i\to\infty}\varphi\left(x_{n_{i}}\right)=\lim\limits _{i\to\infty}\left[\varphi\left(x_{n_{i}}\right)-\varphi_{n_{i}}\left(x_{n_{i}}\right)+y\right]=y.
\]
Applying (a) to $\varphi$, we see that $\varphi^{-1}\in C^{k,\alpha}\left(\overline{U}\right)$.

Next, we shall prove that $\varphi_{n}^{-1}\to\varphi^{-1}$ in $C\left(\overline{U}\right)$-norm.
To do that, it is enough to show that $\left\{ \varphi_{n}^{-1}\right\} _{n=1}^{\infty}$
is relatively compact in $C\left(\overline{U}\right)$ and that $\varphi^{-1}$
is the only possible accumulation point of $\left\{ \varphi_{n}^{-1}\right\} _{n=1}^{\infty}$.
Indeed, $\left\{ \varphi_{n}^{-1}\right\} _{n=1}^{\infty}$ is uniformly
bounded since $\varphi_{n}^{-1}\left(\overline{U}\right)\subset\overline{U}\;\,\left(n\in\mathbb{N}\right)$.
Furthermore, the estimate \eqref{eq:phi_n_Lip} shows that $\left\{ \varphi_{n}^{-1}\right\} _{n=1}^{\infty}$
is equicontinuous on $\overline{U}$. Hence the Arzela-Ascoli theorem
implies the relative compactness of $\left\{ \varphi_{n}^{-1}\right\} _{n=1}^{\infty}$.
If $\psi$ is an accumulation point for $\left\{ \varphi_{n}^{-1}\right\} _{n=1}^{\infty}$,
then there exists a subsequence $\left\{ \varphi_{n_{i}}^{-1}\right\} _{i=1}^{\infty}$
such that $\varphi_{n_{i}}^{-1}\to\psi$ uniformly on $\overline{U}$.
In equalities $\varphi_{n_{i}}^{-1}\circ\varphi_{n_{i}}=I$ and $\varphi_{n_{i}}\circ\varphi_{n_{i}}^{-1}=I$
letting $i\to\infty$, we easily conclude that $\psi\circ\varphi=I$
and $\varphi\circ\psi=I$, therefore $\psi=\varphi^{-1}$.

Using \eqref{eq:inv_deriv}, Corollary \ref{cor:Hoeld_cont_frac}
and the established convergence $\varphi_{n}^{-1}\to\varphi^{-1}$
in $C\left(\overline{U}\right)$-norm, we see that $D\varphi_{n}^{-1}\to D\varphi^{-1}$
uniformly on $\overline{U}$, hence $\varphi_{n}^{-1}\to\varphi^{-1}$
in $C^{1}$-norm. The latter, together with \eqref{eq:inv_deriv}
and Corollary \ref{cor:Hoeld_cont_frac}, implies the convergence
$\varphi_{n}^{-1}\to\varphi^{-1}$ in $C^{2}$-norm, etc. In a finite
number of steps we obtain that $\varphi_{n}^{-1}\to\varphi^{-1}$
in $C^{k}$-norm. The latter, in view of continuity of the embedding
\eqref{eq:cont_deriv_Lip}, implies that $\varphi_{n}^{-1}\to\varphi^{-1}$
in $C^{k-1,1}$-norm. Using this convergence, \eqref{eq:inv_deriv}
and the statement (b) of Lemma \ref{lem:Hoelder_composition}, we
finally conclude that $\varphi_{n}^{-1}\to\varphi^{-1}$ in $C^{k,\alpha}$-norm.
\end{proof}

\section{Suprema of Subgaussian Random Processes}  \label{app:B}

In this section, we collect some known results from probability for the convenience of the reader.

To estimate the sampling error in generative adversarial learning,
we used some results on suprema of subgaussian random processes. We
give the necessary definitions.
\begin{defn}
A random variable $X$ is called $\sigma^{2}$-subgaussian if $\mathbb{E}\left(\left|X\right|\right)<\infty$
and
\[
\mathbb{E}\left[e^{\tau\left[X-\mathbb{E}(x)\right]}\right]\leq e^{\frac{\tau^{2}\sigma^{2}}{2}}\quad\left(\tau\geq0\right).
\]
\end{defn}
The following can be seen as a stochastic variant of the Lipschitz property for stochastic processes: 
\begin{defn}\label{def:Subgaussian}
A real-valued random process $\left\{ X_{t}\right\} _{t\in T}$ on
a metric space $(T,\rho)$ is called subgaussian if $\mathbb{E}\left[X_{t}\right]=0\;\,\left(t\in T\right)$
and the increments $X_{t}-X_{s}$ are $\left[\rho\left(t,s\right)\right]^{2}$-subgaussian.
\end{defn}

Provided stochastic processes are subgaussian, we have the following metric entropy estimate for the expected value of the supremum. Here $N(T,\rho,\varepsilon)$ stands for the covering number of the set $T$, i.e.\ the smallest number of $\rho$-balls of radius $\varepsilon$ covering $T$. 

\begin{thm}[Dudley]
If $\left\{ \pm X_{t}\right\} _{t\in T}$ are continuous subgaussian
processes on the compact metric space $(T,\rho)$, then the estimate
(the so-called entropy bound)
\begin{equation}
\mathbb{E}\left[\sup_{t\in T}\left|X_{t}\right|\right]\leq12\int\limits _{0}^{\infty}\sqrt{\log N\left(T,\rho,\varepsilon\right)}d\varepsilon\label{eq:entr_sum_int}
\end{equation}
holds.
\end{thm}

The proof of \eqref{eq:entr_sum_int} can be found in \cite[Sec. 5.3]{Handel16}.

The following concentration inequality due to C.\ McDiarmid (see \cite{McDiar89}) in addition allows us to control the deviation from the expected value.
This enables us in Sec. \ref{sec:quantitative} to pass from estimates
for the expectation values of the maxima of random processes to estimates
for the maxima of those processes themselves.
\begin{thm}[McDiarmid]\label{thm:MCDiarmid}
 Let $X=(X_{1},X_{2},\ldots,X_{n})$ be a family of independent random
variables with $X_{k}$ taking values in a set $A_{k}$ for each $k$.
Suppose that $f:\prod\limits _{k=1}^{n}A_{k}\to\mathbb{R}$ is a function
with the $(c_{1},\ldots,c_{n})$-bounded differences property: for
each $k=1,2,\ldots,n$ and for any vectors $x,y\in\prod\limits _{k=1}^{n}A_{k}$
that differ only in the $k$-th coordinate, the inequality $f(x)-f\left(y\right)\leq c_{k}$
holds. Then
\begin{equation}
\mathbb{P}\left(f(x)-\mathbb{E}\left[f(x)\right]\geq t\right)\leq e^{-\frac{2t^{2}}{\sum\limits _{i=1}^{n}c_{i}^{2}}}\quad\left(t\geq0\right).\label{eq:McDiar_ineq}
\end{equation}
\end{thm}

\subsubsection*{Acknowledgments}
The authors thank the anonymous referee and Miriam Ackermann for her help in inproving the typoscript. H. A. acknowledges financial support from
Bergisch Smart Mobility funded by the Ministry of Economic Affairs,
Innovation, Digitalization and Energy of the State of North Rhine-Westphalia.
H. G. would like to thank Philipp Petersen for interesting discussions. We also thank the anonymous referee for many useful hints that helped to improve the paper.


\DeclareDelimFormat{finalnamedelim}{\addspace\bibstring{and}\space}

\printbibliography

@inbook{HintonBoltzmann,
author = {Ackley, David H. and Hinton, Geoffrey E. and Sejnowski, Terrence J.},
title = {A Learning Algorithm for Boltzmann Machines},
year = {1988},
isbn = {0893914568},
publisher = {Ablex Publishing Corp.},
address = {USA},
booktitle = {Connectionist Models and Their Implications: Readings from Cognitive Science},
pages = {285--307},
numpages = {23}
}

@InProceedings{pmlr-v70-arjovsky17a,
  title = 	 {{W}asserstein Generative Adversarial Networks},
  author =       {Martin Arjovsky and Soumith Chintala and L{\'e}on Bottou},
  pages = 	 {214--223},
  year = 	 {2017},
  editor = 	 {Doina Precup and Yee Whye Teh},
  volume = 	 {70},
  series = 	 {Proceedings of Machine Learning Research},
  address = 	 {International Convkingention Centre, Sydney, Australia},
  month = 	 {06--11 Aug},
  publisher =    {PMLR},
  url = 	 {http://proceedings.mlr.press/v70/arjovsky17a.html}
}

@article {biau2020,
    AUTHOR = {Biau, G\'{e}rard and Cadre, Beno\^{\i}t and Sangnier, Maxime and
              Tanielian, Ugo},
     TITLE = {Some theoretical properties of {GAN}s},
   JOURNAL = {Ann. Statist.},
  FJOURNAL = {The Annals of Statistics},
    VOLUME = {48},
      YEAR = {2020},
    NUMBER = {3},
     PAGES = {1539--1566},
      ISSN = {0090-5364},
   MRCLASS = {62F12 (62H12 68T01)},
  MRNUMBER = {4124334},
       DOI = {10.1214/19-AOS1858},
       URL = {https://doi.org/10.1214/19-AOS1858},
}

@misc{biau2020wasserstein,
      title={Some Theoretical Insights into Wasserstein GANs}, 
      author={G\'{e}rard Biau and Maxime Sangnier and Ugo Tanielian},
      year={2020},
      eprint={2006.02682},
      archivePrefix={arXiv}
}

@article{brock2018large,
  author    = {Andrew Brock and
               Jeff Donahue and
               Karen Simonyan},
  title     = {Large Scale {GAN} Training for High Fidelity Natural Image Synthesis},
  journal   = {CoRR},
  volume    = {abs/1809.11096},
  year      = {2018},
  url       = {http://arxiv.org/abs/1809.11096},
  archivePrefix = {arXiv},
  eprint    = {1809.11096},
  bibsource = {dblp computer science bibliography, https://dblp.org}
}

@book {Dunf88vol1,
    AUTHOR = {Dunford, Nelson and Schwartz, Jacob T.},
     TITLE = {Linear operators. Part I},
 PUBLISHER = {John Wiley \& Sons, New York},
      YEAR = {1988},
      ISBN = {0-471-60848-3},
   MRCLASS = {47-01 (46-01)},
  MRNUMBER = {1009162},
}

@book{Ferguson.2017,
 author = {Ferguson, Thomas S.},
 year = {1996},
 title = {A Course in Large Sample Theory},
 address = {Boca Raton (Fl)},
 publisher = {Chapman \& Hall},
 isbn = {0-412-04371-8},
}

@Inbook{Ghahramani2004,
author="Ghahramani, Zoubin",
editor="Bousquet, Olivier
and von Luxburg, Ulrike
and R{\"a}tsch, Gunnar",
title="Unsupervised Learning",
bookTitle="Advanced Lectures on Machine Learning: ML Summer Schools 2003, Canberra, Australia, February 2 - 14, 2003, T{\"u}bingen, Germany, August 4--16, 2003, Revised Lectures",
year="2004",
publisher="Springer, Berlin-Heidelberg",
pages="72--112",
isbn="978-3-540-28650-9",
doi="10.1007/978-3-540-28650-9_5",
url="https://doi.org/10.1007/978-3-540-28650-9_5"
}

@book {Gil01,
    AUTHOR = {Gilbarg, David and Trudinger, Neil S.},
     TITLE = {Elliptic partial differential equations of second order},
 PUBLISHER = {Springer, Berlin},
      YEAR = {2001},
      ISBN = {3-540-41160-7},
   MRCLASS = {35-02 (35Jxx)},
  MRNUMBER = {1814364},
}

@book{gine2016mathematical,
  title={Mathematical foundations of infinite-dimensional statistical models},
  author={Gin{\'e}, Evarist and Nickl, Richard},
  volume={40},
  year={2016},
  publisher={Cambridge University Press},
  address = {USA}
}

@incollection{Goodfellow14,
title = {Generative Adversarial Nets},
author = {Goodfellow, Ian and Pouget-Abadie, Jean and Mirza, Mehdi and Xu, Bing and Warde-Farley, David and Ozair, Sherjil and Courville, Aaron and Bengio, Yoshua},
booktitle = {Advances in Neural Information Processing Systems 27},
pages = {2672--2680},
year = {2014},
publisher = {Curran Associates, Montreal},
url = {http://papers.nips.cc/paper/5423-generative-adversarial-nets.pdf}
}

@book {Handel16,
  author={van Handel, Ramon},
  title={Probability in High Dimension},
  year={2016},
  series={APC 550 Lecture Notes},
  publisher={Princeton University}
}

@article{hastings70,
  author = {Hastings, W. K.},
  description = {Origin of Hasting's generalization of Metropolis' algorithm.},
  doi = {10.1093/biomet/57.1.97},
  eprint = {http://biomet.oxfordjournals.org/cgi/reprint/57/1/97.pdf},
  journal = {Biometrika},
  number = 1,
  pages = {97-109},
  title = {Monte Carlo sampling methods using Markov chains and their applications},
  url = {http://biomet.oxfordjournals.org/cgi/content/abstract/57/1/97},
  volume = 57,
  year = 1970
}

@book {Kul97,
    AUTHOR = {Kullback, S.},
     TITLE = {Information theory and statistics},
 PUBLISHER = {Dover Publications, Mineola},
      YEAR = {1997},
      ISBN = {0-486-69684-7},
   MRCLASS = {62B10},
  MRNUMBER = {1461541},
}

@inproceedings{ledig2017photo,
  title={Photo-realistic single image super-resolution using a generative adversarial network},
  author={Ledig, Christian and Theis, Lucas and Husz{\'a}r, Ferenc and Caballero, Jose and Cunningham, Andrew and Acosta, Alejandro and Aitken, Andrew and Tejani, Alykhan and Totz, Johannes and Wang, Zehan and others},
  booktitle={Proceedings of the IEEE conference on computer vision and pattern recognition},
  pages={4681--4690},
  year={2017}
}

@inbook{McDiar89,
 author={McDiarmid, Colin},
 place={Cambridge},
 series={London Mathematical Society Lecture Note Series},
 title={On the method of bounded differences},
 DOI={10.1017/CBO9781107359949.008},
 booktitle={Surveys in Combinatorics, 1989: Invited Papers at the Twelfth British Combinatorial Conference},
 publisher={Cambridge University Press},
 editor={Siemons, J.Editor},
 year={1989},
 pages={148--188},
 collection={London Mathematical Society Lecture Note Series}}

@article{petersen2020equivalence,
  title={Equivalence of approximation by convolutional neural networks and fully-connected networks},
  author={Petersen, Philipp and Voigtlaender, Felix},
  journal={Proceedings of the American Mathematical Society},
  volume={148},
  number={4},
  pages={1567--1581},
  year={2020}
}

@article{petersen2018optimal,
  title={Optimal approximation of piecewise smooth functions using deep ReLU neural networks},
  author={Petersen, Philipp and Voigtlaender, Felix},
  journal={Neural Networks},
  volume={108},
  pages={296--330},
  year={2018},
  publisher={Elsevier}
}

@misc{radford2016unsupervised,
      title={Unsupervised Representation Learning with Deep Convolutional Generative Adversarial Networks}, 
      author={Alec Radford and Luke Metz and Soumith Chintala},
      year={2016},
      eprint={1511.06434},
      archivePrefix={arXiv}
}

@article {ros52,
    AUTHOR = {Rosenblatt, Murray},
     TITLE = {Remarks on a multivariate transformation},
   JOURNAL = {Ann. Math. Statistics},
  FJOURNAL = {Annals of Mathematical Statistics},
    VOLUME = {23},
      YEAR = {1952},
     PAGES = {470--472},
      ISSN = {0003-4851},
   MRCLASS = {62.0X},
  MRNUMBER = {49525},
MRREVIEWER = {S. W. Nash},
       DOI = {10.1214/aoms/1177729394},
       URL = {https://doi.org/10.1214/aoms/1177729394},
}

@book {rud91functional,
    AUTHOR = {Rudin, Walter},
     TITLE = {Functional analysis},
 PUBLISHER = {McGraw-Hill, New York},
      YEAR = {1991},
}

@book {rud76,
    AUTHOR = {Rudin, Walter},
     TITLE = {Principles of Mathematical Analysis},
 PUBLISHER = {McGraw-Hill, New York},
      YEAR = {1976},
   MRCLASS = {26-02},
  MRNUMBER = {0385023},
}

@book{shalev2014understanding,
  title = {Understanding Machine Learning: From Theory to Algorithms},
  author={Shalev-Shwartz, Shai and Ben-David, Shai},
  year={2014},
  isbn = {1107057132},
  publisher={Cambridge University Press},
  address = {USA}
}

@book {Shir16,
    AUTHOR = {Shiryaev, Albert N.},
     TITLE = {Probability-1},
 PUBLISHER = {Springer, New York},
      YEAR = {2016},
  MRNUMBER = {3467826},
}

@misc{shui2020mathcalhdivergence,
      title={Beyond $\mathcal{H}$-Divergence: Domain Adaptation Theory With Jensen-Shannon Divergence}, 
      author={Changjian Shui and Qi Chen and Jun Wen and Fan Zhou and Christian Gagn\'e and Boyu Wang},
      year={2020},
      eprint={2007.15567},
      archivePrefix={arXiv}
}

@inproceedings{singh2018,
 author = {Singh, Shashank and Uppal, Ananya and Li, Boyue and Li, Chun-Liang and Zaheer, Manzil and P{\'o}czos, Barnabas},
 booktitle = {Advances in Neural Information Processing Systems},
 editor = {S. Bengio and H. Wallach and H. Larochelle and K. Grauman and N. Cesa-Bianchi and R. Garnett},
 pages = {10225--10236},
 publisher = {Curran Associates, Inc.},
 title = {Nonparametric Density Estimation under Adversarial Losses},
 url = {https://proceedings.neurips.cc/paper/2018/file/4996dcc43b5be197b5887a4e60817b1c-Paper.pdf},
 volume = {31},
 year = {2018}
}

@book {tay11,
    AUTHOR = {Taylor, Michael E.},
     TITLE = {Partial differential equations I. Basic theory},
 PUBLISHER = {Springer, New York},
      YEAR = {2011},
      ISBN = {978-1-4419-7054-1},
   MRCLASS = {35-01 (46N20 47F05 47N20)},
  MRNUMBER = {2744150},
       DOI = {10.1007/978-1-4419-7055-8},
       URL = {https://doi.org/10.1007/978-1-4419-7055-8},
}

@misc{uppal2020nonparametric,
      title={Nonparametric Density Estimation \& Convergence Rates for GANs under Besov IPM Losses}, 
      author={Ananya Uppal and Shashank Singh and Barnab{\'a}s P{\'o}czos},
      year={2020},
      eprint={1902.03511},
      archivePrefix={arXiv}
}

@book {Vaart96,
    AUTHOR = {van der Vaart, Aad W. and Wellner, Jon A.},
     TITLE = {Weak Convergence and Empirical Processes},
 PUBLISHER = {Springer, New York},
      YEAR = {1996},
      ISBN = {0-387-94640-3},
   MRCLASS = {60F05 (60B12 62G30)},
  MRNUMBER = {1385671},
MRREVIEWER = {Miguel A. Arcones},
       DOI = {10.1007/978-1-4757-2545-2},
       URL = {https://doi.org/10.1007/978-1-4757-2545-2},
}

@article{yarotsky2017error,
  title={Error bounds for approximations with deep {ReLU} networks},
  author={Yarotsky, Dmitry},
  journal={Neural Networks},
  volume={94},
  pages={103--114},
  year={2017},
  publisher={Elsevier}
}

@article{yarotsky2018universal,
  author    = {Dmitry Yarotsky},
  title     = {Universal approximations of invariant maps by neural networks},
  journal   = {CoRR},
  volume    = {abs/1804.10306},
  year      = {2018},
  url       = {http://arxiv.org/abs/1804.10306},
  archivePrefix = {arXiv},
  eprint    = {1804.10306},
  bibsource = {dblp computer science bibliography, https://dblp.org}
}

@INPROCEEDINGS{Zhu.2018,
  author={J. {Zhu} and T. {Park} and P. {Isola} and A. A. {Efros}},
  booktitle={2017 IEEE International Conference on Computer Vision (ICCV)}, 
  title={Unpaired Image-to-Image Translation Using Cycle-Consistent Adversarial Networks},
  year={2017},
  volume={},
  number={},
  pages={2242-2251},
  doi={10.1109/ICCV.2017.244}}

@book{joe2014dependence,
  title={Dependence modeling with copulas},
  author={Joe, Harry},
  year={2014},
  publisher={CRC press}
}

@article{kobyzev2020normalizing,
  title={Normalizing flows: An introduction and review of current methods},
  author={Kobyzev, Ivan and Prince, Simon JD and Brubaker, Marcus A},
  journal={IEEE transactions on pattern analysis and machine intelligence},
  volume={43},
  number={11},
  pages={3964--3979},
  year={2020},
  publisher={IEEE}
}

@article{dinh2014nice,
  title={Nice: Non-linear independent components estimation},
  author={Dinh, Laurent and Krueger, David and Bengio, Yoshua},
  journal={arXiv preprint arXiv:1410.8516},
  year={2014}
}

@article{lala2018evaluation,
  title={Evaluation of mode collapse in generative adversarial networks},
  author={Lala, Sayeri and Shady, Maha and Belyaeva, Anastasiya and Liu, Molei},
  journal={High Performance Extreme Computing},
  year={2018}
}

@book{villani2009optimal,
  title={Optimal transport: old and new},
  author={Villani, C{\'e}dric},
  volume={338},
  year={2009},
  publisher={Springer}
}

@article{brenier1991polar,
  title={Polar factorization and monotone rearrangement of vector-valued functions},
  author={Brenier, Yann},
  journal={Communications on pure and applied mathematics},
  volume={44},
  number={4},
  pages={375--417},
  year={1991},
  publisher={Wiley Online Library}
}

@article{teshima2020coupling,
  title={Coupling-based invertible neural networks are universal diffeomorphism approximators},
  author={Teshima, Takeshi and Ishikawa, Isao and Tojo, Koichi and Oono, Kenta and Ikeda, Masahiro and Sugiyama, Masashi},
  journal={Advances in Neural Information Processing Systems},
  volume={33},
  pages={3362--3373},
  year={2020}
}

@misc{belomnestry2021,
  doi = {10.48550/ARXIV.2102.00199},
   url = {https://arxiv.org/abs/2102.00199},
  author = {Belomestny, Denis and Moulines, Eric and Naumov, Alexey and Puchkin, Nikita and Samsonov, Sergey},
  title = {Rates of convergence for density estimation with GANs},
  publisher = {arXiv},
  year = {2021},
}

@article{chen2020distribution,
  title={Distribution Approximation and Statistical Estimation Guarantees of Generative Adversarial Networks},
  author={Chen, Minshuo and Liao, Wenjing and Zha, Hongyuan and Zhao, Tuo},
  journal={arXiv preprint arXiv:2002.03938},
  year={2020}
}

@article{kingma2019introduction,
  title={An introduction to variational autoencoders},
  author={Kingma, Diederik P and Welling, Max and others},
  journal={Foundations and Trends{\textregistered} in Machine Learning},
  volume={12},
  number={4},
  pages={307--392},
  year={2019},
  publisher={Now Publishers, Inc.}
}
 \addcontentsline{toc}{section}{References}
\end{document}